\documentclass{article}

\usepackage[algo2e,vlined,linesnumbered]{algorithm2e}
\newcounter{alglinenum}

\DontPrintSemicolon

\usepackage{xspace}
\usepackage{amsmath}
\usepackage{framed}
\usepackage{bbm}
\usepackage{algorithm}
\usepackage{float}
\usepackage{natbib}
\usepackage{mathrsfs}
\usepackage{amssymb}
\usepackage{bm}
\usepackage{makecell}
\usepackage{tabulary}
\usepackage{xcolor}
\usepackage{prettyref}
\usepackage{mathtools}
\usepackage{amsmath}
\usepackage{enumitem}
\usepackage{multirow}
\usepackage{nicefrac}
\usepackage{amssymb}
\usepackage{pifont}
\definecolor{Green}{rgb}{0.13, 0.65, 0.3}
\usepackage{colortbl}
\usepackage[]{color-edits}
\addauthor{HL}{red}
\addauthor{cw}{purple}

\DeclareMathOperator*{\argmin}{argmin} 

\newcommand{\command}[1]{{\color{blue}\scalebox{0.9}{\texttt{// #1 }}}}
\newcommand{\prob}{q}
\newcommand{\super}[1]{^{(#1)}}

\newcommand{\e}{\mathrm{e}\xspace}

\newcommand{\SReg}{\text{\rm SReg}}
\newcommand{\DReg}{\text{\rm DReg}}

\newcommand{\ind}{\mathbb{I}}

\newcommand{\xx}{\scalebox{0.75}{$\times$}}
\newcommand{\oo}{\circ}
\newcommand{\broad}{\textsc{Broad-OMD}\xspace}

\newcommand{\LB}{\mathsf{lb}}
\newcommand{\NE}{\mathsf{ne}}

\newcommand{\io}{\iota}
\newcommand{\tdelta}{\tilde{\delta}}

\newcommand{\calF}{\mathcal{F}}

\newcommand{\calI}{\mathcal{I}}

\newcommand{\calJ}{\mathcal{J}}

\newcommand{\calR}{\mathcal{R}}

\newcommand{\otil}{\widetilde{\order}}

\newcommand{\E}{\mathbb{E}}

\newcommand{\order}{O}

\newcommand{\inner}[1]{\langle#1\rangle}

\newcommand{\changedetected}{\textit{ChangeDetected}\xspace}

\newcommand{\term}{\textbf{term}}

\usepackage{tikz}

\newcommand{\nonl}{\renewcommand{\nl}{\let\nl}}

\usepackage{prettyref}
\newcommand{\pref}[1]{\prettyref{#1}}
\newcommand{\pfref}[1]{Proof of \prettyref{#1}}
\newcommand{\savehyperref}[2]{\texorpdfstring{\hyperref[#1]{#2}}{#2}}
\newrefformat{eq}{\savehyperref{#1}{\textup{\eqref{#1}}}}
\newrefformat{ineq}{\savehyperref{#1}{\textup{(\ref*{#1})}}}
\newrefformat{eqn}{\savehyperref{#1}{Equation~\ref*{#1}}}
\newrefformat{lem}{\savehyperref{#1}{Lemma~\ref*{#1}}}
\newrefformat{lemma}{\savehyperref{#1}{Lemma~\ref*{#1}}}
\newrefformat{def}{\savehyperref{#1}{Definition~\ref*{#1}}}
\newrefformat{line}{\savehyperref{#1}{Line~\ref*{#1}}}
\newrefformat{thm}{\savehyperref{#1}{Theorem~\ref*{#1}}}
\newrefformat{corr}{\savehyperref{#1}{Corollary~\ref*{#1}}}
\newrefformat{cor}{\savehyperref{#1}{Corollary~\ref*{#1}}}
\newrefformat{sec}{\savehyperref{#1}{Section~\ref*{#1}}}
\newrefformat{subsec}{\savehyperref{#1}{Section~\ref*{#1}}}
\newrefformat{app}{\savehyperref{#1}{Appendix~\ref*{#1}}}
\newrefformat{assum}{\savehyperref{#1}{Assumption~\ref*{#1}}}
\newrefformat{ex}{\savehyperref{#1}{Example~\ref*{#1}}}
\newrefformat{fig}{\savehyperref{#1}{Figure~\ref*{#1}}}
\newrefformat{alg}{\savehyperref{#1}{Algorithm~\ref*{#1}}}
\newrefformat{rem}{\savehyperref{#1}{Remark~\ref*{#1}}}
\newrefformat{conj}{\savehyperref{#1}{Conjecture~\ref*{#1}}}
\newrefformat{prop}{\savehyperref{#1}{Proposition~\ref*{#1}}}
\newrefformat{proto}{\savehyperref{#1}{Protocol~\ref*{#1}}}
\newrefformat{prob}{\savehyperref{#1}{Problem~\ref*{#1}}}
\newrefformat{claim}{\savehyperref{#1}{Claim~\ref*{#1}}}
\newrefformat{que}{\savehyperref{#1}{Question~\ref*{#1}}}
\newrefformat{op}{\savehyperref{#1}{Open Problem~\ref*{#1}}}
\newrefformat{fn}{\savehyperref{#1}{Footnote~\ref*{#1}}}
\newrefformat{tab}{\savehyperref{#1}{Table~\ref*{#1}}}
\newrefformat{fig}{\savehyperref{#1}{Figure~\ref*{#1}}}
\newrefformat{proc}{\savehyperref{#1}{Procedure~\ref*{#1}}}
\newrefformat{property}{\savehyperref{#1}{Property~\ref*{#1}}}

\usepackage{caption}

\DeclareCaptionFormat{myformat}{#3}

\usepackage{fancybox}

\newcommand{\expalg}{\textup{\textsc{ExpEst}}\xspace}
\newcommand{\expalgCName}{\textbf{Exp}lore and \textbf{Est}imate\xspace}

\newcommand{\alg}{\textsc{DynASORe}\xspace}
\newcommand{\algname}{\textbf{Dyn}amic \textbf{A}nd \textbf{S}tatic \textbf{O}ptimal \textbf{Re}gret \xspace}

\usepackage{titlesec}

\DeclarePairedDelimiter{\abs}{\lvert}{\rvert} %
\DeclarePairedDelimiter{\brk}{[}{]}
\DeclarePairedDelimiter{\crl}{\{}{\}}
\DeclarePairedDelimiter{\prn}{(}{)}
\usepackage{scalerel}[2016/12/29]

\usepackage{color-edits}
\addauthor{jq}{yellow!60!black}

\newcommand{\ldef}{\vcentcolon=}

\newcommand{\vhat}{\hat{v}}
\newcommand{\lhat}{\hat{\ell}}
\newcommand{\Lhat}{\hat{L}}
\newcommand{\En}{\mathbb{E}}

\newcommand{\cH}{\mathcal{H}}

\newcommand{\bR}{\mathbb{R}}

\newcommand{\exprate}{\rho}

\usepackage[utf8]{inputenc}
\usepackage[T1]{fontenc}
\usepackage{graphicx}
\usepackage{amsmath,amsfonts,amssymb,amsthm}
\usepackage{natbib}

\usepackage{mathtools}
\usepackage{mathrsfs}
\usepackage{xcolor}
\usepackage{hyperref}
\usepackage{booktabs}
\usepackage{bbold}
\usepackage{bm}
\usepackage{color}
\usepackage[font=footnotesize]{caption}
\usepackage{enumitem}
\usepackage{empheq}
\usepackage{framed}
\usepackage{comment}
\usepackage{fullpage}
\usepackage{soul}
\usepackage{dirtytalk}
\usepackage{dsfont}
\usepackage{bbm}
\usepackage{cleveref}

\definecolor{myred}{HTML}{880000}
\definecolor{mygreen}{HTML}{008800}
\definecolor{myblue}{HTML}{000088}
\definecolor{linkblue}{HTML}{0000BB}

\hypersetup{
  colorlinks,
  linkcolor={blue},
  citecolor={blue},
  urlcolor={linkblue}
}

\renewcommand{\leq}{\leqslant}
\renewcommand{\geq}{\geqslant}
\renewcommand{\le}{\leqslant}

\renewcommand{\ln}{\log}

\usepackage{xspace}

\newtheorem{theorem}{Theorem}
\newtheorem*{theorem*}{Theorem}
\newtheorem{lemma}{Lemma}

\theoremstyle{definition}

\theoremstyle{remark}

\title{Achieving Optimal Static and Dynamic Regret Simultaneously \\in Bandits with Deterministic Losses}
\author{ Jian Qian\thanks{The University of Hong Kong. \texttt{jianqian@hku.hk}.}
\ \ \ \ \ \  Chen-Yu Wei\thanks{ University of Virginia. \texttt{ chenyu.wei@virginia.edu}.}
}
\date{}

\begin{document}
\maketitle

\begin{abstract}%
In adversarial multi-armed bandits, two performance measures are commonly used: static regret, which compares the learner to the best fixed arm, and dynamic regret, which compares it to the best sequence of arms. While optimal algorithms are known for each measure individually, there is no known algorithm achieving optimal bounds for both simultaneously. \cite{marinov2021pareto} first showed that such simultaneous optimality is impossible against an \emph{adaptive} adversary. Our work takes a first step to demonstrate its possibility against an \emph{oblivious} adversary when losses are \emph{deterministic}. 
First, we extend the impossibility result of \cite{marinov2021pareto} to the case of deterministic losses. Then, we present an algorithm achieving optimal static and dynamic regret simultaneously against an oblivious adversary. Together, they reveal a fundamental separation between adaptive and oblivious adversaries when multiple regret benchmarks are considered simultaneously. It also provides new insight into the long open problem of simultaneously achieving optimal regret against switching benchmarks of different numbers of switches.

Our algorithm uses negative static regret to compensate for the exploration overhead incurred when controlling dynamic regret, and leverages Blackwell approachability to jointly control both regrets. 
This yields a new model selection procedure for bandits that may be of independent interest.

\end{abstract}

\section{Introduction} \label{sec: introduction}

In multi-armed bandits, the learner interacts with an adversary over a total number of $T$ time steps. At each time step, the adversary assigns losses to actions (arms), and the learner selects one, observing only its associated loss. The learner's objective is to minimize regret, defined as the gap between the learner's cumulative loss and the cumulative loss of a benchmark policy.  


The choice of the benchmark policy usually depends on the learner's belief about the type of the adversary, and this influences the algorithm design. For example, if the adversary is mostly stationary with only a few corruptions, it is natural to consider the \emph{static regret}, in which the benchmark policy chooses the best fixed arm across time. For this setting, algorithms such as EXP3 \citep{auer2002nonstochastic} or Tsallis-INF \citep{audibert2009minimax, zimmert2021tsallis} are able to achieve the optimal static regret. 
On the other hand, if the adversary is nonstationary, it is more 
reasonable to consider the \emph{dynamic regret}, in which the benchmark policy chooses the best sequence of arms that tracks the changes of the losses. For this setting, one can adopt the EXP3.S algorithm by \cite{auer2002nonstochastic}, or discounting or restarting-based algorithms by, e.g., \cite{cheung2019learning} and \cite{auer2019adaptively}, to achieve the optimal dynamic regret. These algorithms usually involve mechanisms that help the learner forget outdated data. While the benchmark in dynamic regret is more powerful than in static regret, the optimal static regret bound is smaller, and thus a learner focusing on static regret minimization can still outperform one that focuses on dynamic regret. From an algorithm design perspective, algorithms focusing on static regret minimization typically do not forget past data, or utilize smaller learning rate and exploration rate. This allows them to be more robust against adversarial corruptions without being overly sensitive to them. 

The two distinct benchmarks and the their corresponding algorithm designs (i.e., to forget or not to forget past data, learn faster or learn slower, explore more or explore less) pose a challenging decision for the learner: which algorithmic approach should they adopt if the underlying type of adversary is unknown? In fact, it might not even be possible to unambiguously classify the adversary: a stationary adversary with no more than $S/2$ corruptions (where each corruption lasts for several time steps) looks exactly the same as a nonstationary adversary with no more than $S-1$ changes. In this regard, a more reasonable goal is to be optimal against both types of benchmarks \emph{simultaneously}, without bothering to guess the type of the adversary. However, is this possible, and can existing algorithms achieve this? 

This leads to our main question. In multi-armed bandits, $\otil(\sqrt{AT})$ and  $\otil(\sqrt{SAT})$ are known to be the optimal static and dynamic regret bounds, respectively, where $S-1$ is the number of times the losses change, $A$ is the number of actions, and $T$ is the number of time steps the learner interacts with the adversary. We ask: \emph{is there an algorithm that can achieve $\otil(\sqrt{AT})$ static regret and $\otil(\sqrt{SAT})$ dynamic regret simultaneously?} Technically, if we adopt the EXP3.S algorithm, then achieving $\otil(\sqrt{AT})$ static regret requires learning rate and exploration rate of roughly $\sqrt{A/T}$, while achieving $\otil(\sqrt{SAT})$ dynamic regret requires them to be roughly $\sqrt{SA/T}$. The core challenge of our question lies in how to seamlessly tune the hyperparameter value so that the performance is no worse than either value running alone. 

To our knowledge, prior to our work, no algorithm can achieve optimal static and dynamic regret bounds simultaneously. In fact, the room for such an algorithm is small---\cite{marinov2021pareto} showed that this is impossible if the adversary is adaptive, that is, the loss in a certain time step can depend on the learner's previous choices. 
However, their lower bound does not preclude the possibility when the adversary is oblivious. For oblivious adversary, the best prior result is by \cite{cheung2021hedging}, who proposed a bandit-over-bandit framework that achieves sub-optimal simultaneous bounds: with EXP3.S as the base algorithm, their algorithm guarantees $\otil(A^{1/4}T^{3/4})$ static regret and $\otil(A^{1/4}T^{3/4} + \sqrt{SAT})$ dynamic regret simultaneously.

In this paper, we take a first step in showing the possibility of achieving simultaneous optimal bounds in a special case: if the losses are \emph{deterministic} and $S$ is \emph{known}, then $\otil(\sqrt{AT})$ static regret and $\otil(\sqrt{SAT})$ dynamic regret can be achieved simultaneously. 
To demonstrate these additional assumptions do not trivialize the problem, we strengthen the lower bound of \cite{marinov2021pareto} by showing that their impossibility result still holds even with deterministic losses and known~$S$. Putting them together, we demonstrate a strong separation between oblivious and adaptive adversary for bandits when static and dynamic regrets are considered simultaneously.  

Our algorithm is based on Blackwell approachability. Blackwell approachability reduces vector-valued (multi-objective) online learning to standard scalar-valued online learning \citep{blackwell1956analog,  abernethy2011blackwell}. The payoff in the $i$-th coordinate in the induced problem corresponds to the learner’s ``regret'' (violation) in the $i$-th objective (constraint) in the original problem. In our case, the induced problem has $A+1$ dimensions where the first $A$ correspond to static regret against individual arms, and the $(A+1)$-th corresponds to the dynamic regret. To operate in this ``regret space'', approachability-based algorithms \citep{hart2000simple, luo2015achieving} require the learner to estimate the ``regret vector'' at each round. Being straightforward in the full-information setting, it requires extra efforts in the partial-information setting, and generally achieves a worse rate than in the full-information case \citep{mannor2014set, kwon2017online}.  We develop several tailored techniques to achieve the optimal rates. More related work is discussed in \pref{app: related work}.

\section{Preliminaries} \label{sec: prelim}

\paragraph{Notation}
For any $N\in\mathbb{N}$, $[N] \triangleq  \crl{1,\dots,N}$. $\ind\{\cdot\}$ denotes the indicator function. $\otil(\cdot)$ denotes $O(\cdot)$ with logarithmic factors ignored. $v(a)$ denotes the $a$-th coordinate of a vecotr $v$.

\paragraph{Problem Setup}
We consider the switching bandit problem with deterministic losses. Let $A$ be the number of actions and $[A]=\{1,2,\ldots, A\}$ be the action set. The interaction between the learner and the adversary lasts for $T$ times steps. At each time step $t\in [T]$, the adversary assigns the losses $c_t(a)$ to each arm $a\in [A]$. Then the learner selects an arm $a_t$ and observes the associated loss $c_t(a_t)$. By ``deterministic losses,'' we mean that the learner directly observes $c_t(a_t)$ rather than a noisy realization of it.\footnote{For static regret minimization in adversarial bandits \citep{auer2002nonstochastic}, it is standard to assume the learner directly observes $c_t(a_t)$. For dynamic regret minimization \citep{auer2019adaptively}, it is usually assumed that the learner observes $c_t(a_t) + w_t$ for some zero-mean noise $w_t$. In this work, we assume $w_t$ is always zero.}  The static regret with respect to any action $a\in [A]$ is
\begin{align*}
    \SReg(a) \ldef \sum\nolimits_{t=1}^T \left(c_t(a_t) - c_t(a)\right).
\end{align*}
The two objectives of the static regret and the dynamic regret are defined respectively as
\begin{align*}
\SReg \ldef \max_{a\in [A]} \SReg(a) \quad \mbox{and} \quad \DReg \ldef \sum\nolimits_{t=1}^{T} \Big(c_t(a_t) -\min_{a\in [A]} c_t(a) \Big).
\end{align*}
We assume further that the adversary is nonstationary but changes the loss no more than $S-1$ times, where $S\ldef 1+\sum_{t=2}^T \ind\{c_t\neq c_{t-1}\}$ is known. 
We note that  an adaptive adversary can assign the losses $c_t$ depending on all the interactions before time step $t$, whereas an oblivious adversary chooses all the losses $c_t$ beforehand for all $t\in [T]$.

We mention in passing the complementary lower bound result regarding the switching bandit problem with an oblivious adversary and with deterministic losses. 
\begin{theorem}[Theorem~4.1 of \cite{wei2016tracking}]
For any $S\geq 2$, there is a deterministic multi-armed bandit problem with no more than $S-1$ switches such that any algorithm, even with knowledge of $S$, must suffer $\DReg \geq \Omega(\sqrt{SAT})$ against an oblivious adversary.  
\end{theorem}

\paragraph{Paper organization}
In \pref{sec: lower bound}, we show the impossibility for achieving simultaneous optimality with an adaptive adversary.
In \pref{sec: high level}, we explain where the opportunity lies when facing an oblivious adversary. In \pref{sec: warmup}, we consider a special case akin to the lower bound construction in \pref{sec: high level}, but with the adaptivity of the adversary removed, and present an algorithm that achieves simultaneous optimality. Finally, in \pref{sec: full algo}, we introduce our main algorithm and main theorem.

\section{A Lower Bound for Adaptive Adversary} \label{sec: lower bound}
\newcommand{\PHone}{[PH1]}
\newcommand{\PHtwo}{N}
\newcommand{\PHthree}{M}

In this section, we introduce the lower bounds regarding the switching bandit problem as formulated in \pref{sec: prelim}. We first strengthen the hardness result with an adaptive adversary from \cite{marinov2021pareto} to our deterministic losses case. We then present the known lower-bound results for static regret and dynamic regret with an oblivious adversary. 

\begin{theorem}[Hardness with an adaptive adversary]\label{thm: lower bound general S}
    For any integer $S\geq 2$, there is an adaptive adversary with no more than $S$ swithes such that any algorithm must suffer either $\SReg\geq \Omega(S^\alpha \sqrt{T})$ or $\DReg\geq \Omega(S^{1-\alpha}\sqrt{T})$ for any $\alpha\in \bR$ such that
    $\max\crl{S^{1-\alpha},S^\alpha} \leq \sqrt{T}$.
      
\end{theorem}



\begin{proof}\textbf{sketch}\quad
    The full proof is in \pref{app: lower bound proof}. We consider a switching bandit problem with two arms. For arm 1, the loss is fixed across time to be $1/2$. For arm 2, the loss starts from $1$ and will be changed to $0$ and then back to $1$ from time to time. We divide the time steps into $S/2$ epochs of equal length $2T/S$. The oblivious adversary we construct decides whether or not to switch the loss of arm 2 from $1$ to $0$ and then back to $1$ for a given length of $M$ in each epoch. The oblivious adversary performs the switch when the learner's expected number of draws to arm 2 in the current epoch is less than $N=T/(SM)$. Furthermore, it performs the switches in the sub-interval of length $M$ where the expected number of draws to arm 2 is the least in the current epoch. By definition, the expected number of draws to arm~2 in the aforementioned sub-interval is upper bounded by $N\times \frac{M}{2T/S}=1/2$, implying that with a probability at least $1/2$ there is no draws to arm~2 in this sub-interval and thus the dynamic regret in the epoch is $M/2$. Now we build on top of this oblivious adversary an adaptive adversary with the twist that if ever the learner pulls arm 2 in the sub-interval of switch, the adaptive adversary immediately switches the loss of arm 2 back to $1$. Let $E$ be the total number of epochs where the adversary decides not to perform the switch. Then if $E > 3S/8$, the static regret against arm 1 is at least $\Omega (SN)$. Else then $S/2-E >  S/8$ and the dynamic regret is at least $\Omega(SM)$. Since we have set $MN=T/S$, we have $\SReg\cdot \DReg \geq \Omega(ST)$. 
\end{proof}

\pref{thm: lower bound general S} strengthens the hardness result of \cite{marinov2021pareto} whose construction requires three arms and stochastic losses. \pref{thm: lower bound general S} asserts that if an algorithm achieves $\SReg = O(\sqrt{T})$, then it must suffer $\DReg = \Omega(S\sqrt{T})$. Or, if an algorithm achieves $\DReg = O(\sqrt{ST})$, then it must suffer $\SReg = \Omega(\sqrt{ST})$.

\pref{thm: lower bound general S}'s hardness relies on the adversary's adaptivity to immediately switch back the loss of arm 2 after the learner detects its loss change. If the adversary is unable to do so, there is opportunity for the learner. Specifically, if the adversary fails to promptly revert after a switch detection, the learner incurs \emph{negative static regret} relative to arm 1. We elaborate this in the next section.

\section{Opportunities with an Oblivious Adversary} \label{sec: high level}


The central challenge in achieving optimal static and dynamic regret simultaneously lies in the selection of the exploration rate. On the one hand,
an exploration rate of $\sqrt{S/T}$ is generally necessary to achieve
$\sqrt{ST}$ dynamic regret: if the exploration rate is smaller than this, changes on rarely
played arms will be detected with more delay, during which the learner follows an outdated policy and incurs excess dynamic regret.  On the other hand, exploration rate beyond $1/\sqrt{T}$ introduces overhead that can
cause the static regret to exceed $\sqrt{T}$. From this viewpoint, there does not seem to be any single exploration rate that allows the learner to achieve $\sqrt{ST}$ dynamic regret and $\sqrt{T}$ static regret simultaneously. However, two insights described below resolve this dilemma. 


One insight is that exploration rate over $1/\sqrt{T}$ is not purely harmful for static regret. Although exploration incurs additional cost, it
can also reveal unexpectedly good performance of rarely played arms---\emph{provided that the adversary changes some rarely played arms to be good}. This creates~the learner's \emph{negative static regret} against fixed arms, which can compensate for the exploration overhead. 

A complementary insight is that exploration rate less than $\sqrt{S/T}$ may not always cause excess dynamic regret---\emph{provided that adversary changes none of the rarely played arms to be good}. Indeed, the sole purpose of
using a larger exploration rate in controlling dynamic regret is to detect such changes. When no such changes
occur, under-exploration does not lead to additional dynamic regret. 

Taken together, these insights suggest a win-win strategy for the learner
\emph{if she knew in advance whether the adversary would change any rarely played
arms to be good}. If the adversary does so, the learner could employ a larger exploration
rate, simultaneously controlling dynamic regret through reduced detection delay and
gaining negative static regret to compensate for the exploration cost. If not,
the learner could just use a smaller exploration rate, as all rarely played
arms would remain suboptimal and lead to no excess dynamic regret.

In reality, of course, the learner does not know in
advance how the adversary is going to act. This is precisely where Blackwell approachability (or no-regret learning) becomes useful.
At a high level, if the learner has a strategy that achieves the desired objective
against each possible adversary strategy when that strategy is known, then
Blackwell’s framework guarantees that
the same objective can be achieved without such prior knowledge in the long run by repeated play. 

Finally, we note that an adaptive adversary, as constructed in \pref{thm: lower bound general S}, is able to break this win-win strategy: he can change some rarely played arms to be good to induce dynamic regret through detection delay, but immediately revert the losses once the learner successfully detects the change, preventing the learner from gaining negative static regret.

\section{Warm-Up: Beating Oblivious Adversary in the Lower Bound Example} \label{sec: warmup}

In this section, we describe an algorithm that achieves optimal static and dynamic regret in the example used to construct the lower bound in \pref{sec: lower bound}, though we remove the ability of the adversary to instantly revert the loss vector after learner's detection. This warm-up is simple but demonstrates many key ideas of the general algorithm to be discussed in \pref{sec: full algo}. 

\paragraph{Problem setup} We consider an oblivious adversary that generates losses in the following
restricted manner. The time horizon is divided into $K=\sqrt{T}$ epochs, each
of length $\sqrt{T}$. For each epoch $k$, denoted by
$\calI_k=[(k-1)\sqrt{T}+1,\,k\sqrt{T}]$, the adversary selects a sub-interval
$\calJ_k\subset\calI_k$ of length $|\calJ_k|\le \frac{\sqrt{T}}{2}=\frac{|\calI_k|}{2}$. The loss
vectors are defined as $c_t=(\frac{1}{2},0)$ for $t\in\calJ_k$ and
$c_t=(\frac{1}{2},1)$ for $t\in\calI_k\setminus\calJ_k$. The learner’s goal is
to achieve static regret $\SReg=O(\sqrt{T})$ and dynamic regret
$\DReg=O(T^{3/4})$ simultaneously.  
Since $|\calJ_k|\leq \frac{1}{2}|\calI_k|$ for all $k$, arm~1 is the globally optimal arm. Therefore, it suffices to measure static regret with
respect to arm~1. 

Recall that if the adversary is adaptive (i.e., able to switch the loss vector back immediately after the learner detects a change), then the lower
bound in \pref{thm: lower bound general S} (with $S=2\sqrt{T}$,
$M=\sqrt{T}/2$, and $\alpha=1/4$) implies that any learner must
suffer either $\SReg=\Omega(T^{5/8})$ or $\DReg=\Omega(T^{7/8})$.

\paragraph{Epoch-level strategies}
At the beginning of each epoch, the learner chooses between the following two
strategies, denoted by $\xx$ and $\oo$:

\refstepcounter{equation}\label{eq: def of o and x}
\vspace*{3pt}
\noindent
\fbox{%
\begin{minipage}{0.98\linewidth}
\begin{tabular}{@{}p{0.7em}@{:\ \ \ }p{\dimexpr\linewidth-3em\relax}@{}}
  \raggedleft$\xx$ & Play arm~1 throughout the entire epoch. \\
  \raggedleft$\oo$ & In each round, play arm~2 with probability
  $\epsilon=T^{-1/4}$ and arm~1 otherwise.
  Upon observing $c_t(2)=0$, switch to
  arm~2 and continue playing it until $c_t(2)$ returns to~$1$, after which the
  learner plays arm~1 for the remainder of the epoch. \hfill(\theequation)
\label{eq: def of o and x}
\end{tabular}
\end{minipage}
}


\subsection{Regret Upper Bounds under $\xx$ and $\oo$} \label{sec: loss construction warmup}
In each epoch $k$, the learner chooses to execute either $\xx$ or $\oo$. Let $L_k=|\calJ_k|$, and denote by $\SReg_{k,\xx}$, $\DReg_{k,\xx}$, $\SReg_{k,\oo}$, $\DReg_{k,\oo}$ the static regret (against arm~1) and the dynamic regret in epoch $k$ under the strategies $\xx$ and $\oo$, respectively. 

Under strategy $\xx$, arm~1 is selected throughout the epoch, yielding
\begin{align}
    \SReg_{k,\xx} = 0, \qquad
    \DReg_{k,\xx} = \tfrac{1}{2}L_k.   \label{eq: Sreg bound}
\end{align}
 
Under strategy $\oo$, let $\hat{L}_k \le L_k$ denote the number of rounds in
$\calJ_k$ during which the learner plays arm~2. These rounds provide negative
static regret. A direct calculation shows that
\begin{align*}
    \E[\SReg_{k,\oo}] \leq \tfrac{1}{2} T^{1/4} -\tfrac{1}{2} \E[\hat{L}_k] + \tfrac{1}{2}, \qquad
    \E[\DReg_{k,\oo}] \leq \tfrac{1}{2}T^{1/4} + \tfrac{1}{2}(L_k - \E[\hat{L}_k]+1), 
\end{align*}
where in the bound of $\SReg_{k,\oo}$, the term
$\tfrac{1}{2}T^{1/4} = \tfrac{1}{2}\epsilon\sqrt{T}$
accounts for exploration,
$-\tfrac{1}{2}\hat{L}_k$ is the negative regret accumulated when
arm~2's loss is zero, and the final $\tfrac{1}{2}$ comes from the switch
point at which the adversary reverts the loss vector. To bound $\DReg_{k,\oo}$, we use the fact that the difference between dynamic and static regret is always $\tfrac{1}{2}L_k$, regardless of whether $\xx$ or $\oo$ is
chosen. 
Since $L_k - \hat{L}_k$ equals the number of failures (drawing arm~1) before the first success (drawing arm~2) 
in a Bernoulli process with success probability $\epsilon$, we have
$L_k - \E[\hat{L}_k] + 1 \le \epsilon^{-1} = T^{1/4}$. Consequently, 
\begin{align}
    \E[\SReg_{k,\oo}] \leq T^{1/4} - \tfrac{1}{2}L_k, \qquad 
    \E[\DReg_{k,\oo}] \leq T^{1/4}.  \label{eq: Dreg bound}
\end{align}

\subsection{A Multi-Objective Formulation}
Let $\overline{\SReg}_{k,\xx}$, $\overline{\DReg}_{k,\xx}$, $\overline{\SReg}_{k,\oo}$, $\overline{\DReg}_{k,\oo}$ be the respective upper bounds for $\E[\SReg_{k,\xx}]$, $\E[\DReg_{k,\xx}]$, $\E[\SReg_{k,\oo}]$, $\E[\DReg_{k,\oo}]$ identified in \pref{eq: Sreg bound} and~\pref{eq: Dreg bound}. We define auxiliary vectors 
\begin{align}
     \ell_{k,\xx} =  \begin{bmatrix} \sqrt{T} \cdot \overline{\SReg}_{k,\xx} \\ T^{1/4}\cdot \overline{\DReg}_{k,\xx} \end{bmatrix}  = \begin{bmatrix}
          0 \\
          \frac{1}{2}T^{1/4}L_k
     \end{bmatrix}, \quad 
     \ell_{k,\oo} = \begin{bmatrix} \sqrt{T} \cdot \overline{\SReg}_{k,\oo} \\ T^{1/4}\cdot \overline{\DReg}_{k,\oo} \end{bmatrix} = \begin{bmatrix}
          T^{3/4} - \frac{1}{2}\sqrt{T}L_k \\
          \sqrt{T}
     \end{bmatrix} \label{eq: x and o losses definitions} 
\end{align}
where we plug in the upper bounds in \pref{eq: Sreg bound} and~\pref{eq: Dreg bound}, and the scaling is chosen so that it suffices to control both objectives to the same $O(T)$ order. 

Let $p_{k,\xx}$ and $p_{k,\oo}$ denote the probabilities of selecting strategies
$\xx$ and $\oo$ in epoch $k$. Then the learner's total static and dynamic regret satisfy
\begin{align*}
\E[\SReg] &= \textstyle\E\left[\sum_{k=1}^K  \left(p_{k,\xx} \SReg_{k,\xx} + p_{k,\oo} \SReg_{k,\oo}  \right)\right] \leq \E\left[\sum_{k=1}^K  \left(p_{k,\xx} \overline{\SReg}_{k,\xx} + p_{k,\oo} \overline{\SReg}_{k,\oo}  \right)\right],    \\
    \E[\DReg] &= \textstyle \E\left[\sum_{k=1}^K  \left(p_{k,\xx} \DReg_{k,\xx} + p_{k,\oo} \DReg_{k,\oo}  \right)\right] \leq \E\left[\sum_{k=1}^K  \left(p_{k,\xx} \overline{\DReg}_{k,\xx} + p_{k,\oo} \overline{\DReg}_{k,\oo}  \right)\right].   
\end{align*}
With the above definitions, achieving
$\E[\SReg]=O(\sqrt{T})$ and
$\E[\DReg]=O(T^{3/4})$
can be written as the multi-objective condition: 
\begin{align}
   \E\left[\sum_{k=1}^{K}  \left(p_{k,\xx}\ell_{k,\xx} + p_{k,\oo}\ell_{k,\oo}\right)\right] \leq \begin{bmatrix}
         O(T) \\
         O(T)
    \end{bmatrix}.   \label{eq: target bound}
\end{align}
\paragraph{Checking approachability under full information} At this point, we would like to check whether \pref{eq: target bound} is even \emph{possible}. Assume for a moment the learner can observe $L_k$ (and thus $\ell_{k,\oo}$ and $\ell_{k,\xx}$) after each epoch~$k$. Then the theorem of Blackwell approachability (e.g., Theorem~3 of \cite{abernethy2011blackwell}) states that \pref{eq: target bound} is achievable if and only if for any possible $(\ell_{k,\xx}, \ell_{k,\oo})$, there exists $(p_{\xx}, p_{\oo})\in\Delta_2$ such that $p_{\xx}\ell_{k,\xx} + p_{\oo}\ell_{k,\oo}\leq \frac{1}{K}[O(T), O(T)] = [O(\sqrt{T}), O(\sqrt{T})]$. This reduces multi-round approachability to single-round satisfiability. To verify this condition, suppose $(\ell_{k,\xx}, \ell_{k,\oo})$ are defined as \pref{eq: x and o losses definitions} for some $L_k$. Then choosing $(p_{\xx}, p_{\oo}) = \left(\ind\{L_k\leq 2T^{1/4}\}, \ind\{L_k > 2T^{1/4}\}\right)$ yields 
\begin{align*}
    p_{\xx}\ell_{k,\xx} + p_{\oo}\ell_{k,\oo} = \begin{bmatrix}
        0 \\ \frac{1}{2}T^{1/4}L_k \ind\{L_k\leq 2T^{1/4}\} 
    \end{bmatrix} + \begin{bmatrix}
        \ind\{L_k>2T^{1/4}\}(T^{3/4} - \frac{1}{2}\sqrt{T}L_k) \\ \ind\{L_k>2T^{1/4}\}\sqrt{T} 
    \end{bmatrix}
    \leq \begin{bmatrix}
        0 \\ \sqrt{T} 
    \end{bmatrix}, 
\end{align*}
satisfying the desired condition. This gives an initial evidence that simultaneous
optimality of static and dynamic regret is possible. Of course, more challenges arise as the learner cannot access the exact value of $L_k$. 
Next, we construct suitable estimators and give
a concrete online algorithm.



\subsection{A Recipe for Approachability Algorithm Design}\label{sec: approachability warmup} 

Since $L_k$ is not directly observable and can only be estimated when the
learner chooses $\oo$, we introduce the following loss estimators for
\pref{eq: x and o losses definitions}: 
\begin{align}
     \hat{\ell}_{k,\xx} = \frac{\ind\{i_k=\oo\}}{p_{k,\oo}}\begin{bmatrix}
          0 \\
          \frac{1}{2}T^{1/4}\hat{L}_k
     \end{bmatrix}, \qquad 
     \hat{\ell}_{k,\oo} = \frac{\ind\{i_k=\oo\}}{p_{k,\oo}}\begin{bmatrix}
          T^{3/4} - \frac{1}{2}{\color{red}u}\sqrt{T}\hat{L}_k \\
          \sqrt{T}
     \end{bmatrix},   \label{eq: warmup estimator}
\end{align}
where $i_k\in\{\xx,\oo\}$ is the strategy sampled from $i_k\sim (p_{k,\xx}, p_{k,\oo})$ and $\hat{L}_k$
is as defined in \pref{sec: loss construction warmup}. The constant
$u=\frac{1}{3}$ is a slack parameter whose role will become clear
later.
Our goal is to choose
$(p_{k,\xx},p_{k,\oo})$ so that $\E[\sum_{k=1}^{K} (p_{k,\xx}\hat{\ell}_{k,\xx} + p_{k,\oo}\hat{\ell}_{k,\oo})]$ is bounded elementwisely by $O(T)$. 
Below, we describe a recipe inspired by Blackwell approachability to achieve this: 

\begin{enumerate}[leftmargin=16pt, itemsep=0pt, topsep=0pt, parsep=0pt, partopsep=0pt]
\item
\emph{(Defining the regret vector)}
Fix an epoch $k$ and suppose the probabilities
$(p_{k,\xx},p_{k,\oo})$ are given.
The learner samples $i_k\sim (p_{k,\xx}, p_{k,\oo})$,
executes the corresponding strategy in epoch $k$, and builds the estimators $(\hat{\ell}_{k,\xx}, \hat{\ell}_{k,\oo})$ as in
\pref{eq: warmup estimator}.
Then define the regret vector $\hat{v}_k = p_{k,\xx}\hat{\ell}_{k,\xx} + p_{k,\oo}\hat{\ell}_{k,\oo}$. 

    \item \emph{(Induced online learning problem)} We introduce a sequence
$\{\theta_k\}_{k=1}^K$ with $\theta_k\in\Delta_2$.
After receiving the regret vector $\hat v_k\in\mathbb{R}^2$ from Step~1, the learner updates $\theta_k$
using any online learning algorithm with $\hat v_k$ as the \emph{payoff} vector. 
Assume this algorithm guarantees
$
\sum_{k=1}^K
(\hat v_k(a) - \theta_k^\top \hat v_k)
\leq
\mathcal R(a)
$ for each $a\in\{1,2\}$, 
where $\mathcal R(a)$ is a regret bound with respect to comparator $a$.
    \item \emph{(Controlling the objectives)} The regret bound in Step~2 gives us $\sum_{k=1}^K \hat{v}_{k}(a) \leq \mathcal{R}(a) + \sum_{k=1}^K  \theta_k^\top \hat{v}_k$ for all $a$. Observe that the left-hand side is the $a$-th coordinate of $\sum_{k=1}^K (p_{k,\xx}\hat{\ell}_{k,\xx} + p_{k,\oo}\hat{\ell}_{k,\oo})$, which is exactly the $a$-th objective we aim to control. Based on previous discussions, it suffices to bound $\E[\calR(a)+\sum_{k=1}^K \theta_k^\top \hat{v}_k]$ by $O(T)$ for both $a\in\{1,2\}$. 
    \item \emph{(Choice of probabilities)} Given $\theta_k$ at the beginning of epoch $k$, we would like to choose 
$(p_{k,\xx},p_{k,\oo})$ to control the
magnitude of $\theta_k^\top \hat v_k$ so that $\E[\sum_{k=1}^K \theta_k^\top \vhat_k]$ is bounded by $O(T)$.  The criteria to choose $(p_{k,\xx}, p_{k,\oo})$ is that it should make $\theta_k^\top \hat v_k =  p_{k,\xx}(\theta_k^\top \hat{\ell}_{k,\xx}) +  p_{k,\oo}(\theta_k^\top \hat{\ell}_{k,\oo})$ small under the \emph{worst-case} realization of $(\hat{\ell}_{k,\xx}, \hat{\ell}_{k,\oo})$. 
    
\end{enumerate}
With the recipe above, our problem reduces to the following two sub-problems: 1) given any $\theta_k\in \Delta_2$, generate $(p_{k,\xx}, p_{k,\oo})$ to make  $\theta_k^\top \hat v_k$ small under the worst-case realization of $(\hat{\ell}_{k,\xx}, \hat{\ell}_{k,\oo})$, and 2) design an online algorithm that ensures $\sum_{k=1}^{K} (\hat{v}_{k}(a) - \theta_k^\top \hat{v}_k) \leq \mathcal{R}(a)$ is small. These two jointly ensure $\calR(a) + \sum_{k=1}^K \theta_k^\top \hat{v}_k$ is small for $a\in\{1,2\}$, which in turn controls all objectives according to the Step~3 above. In the next two subsections, we further look into these two sub-problems.  

\subsection{Choice of $(p_{k,\xx}, p_{k,\oo})$ and Bound on $\sum_k \theta_k^\top \vhat_k$}\label{sec: subprob 1}  Substituting the definitions of $\hat{\ell}_{k,\xx}$ and $\hat{\ell}_{k,\oo}$ in \pref{eq: warmup estimator} into $\theta_k^\top \hat v_k =  p_{k,\xx}(\theta_k^\top \hat{\ell}_{k,\xx}) +  p_{k,\oo}(\theta_k^\top \hat{\ell}_{k,\oo})$, we get
\begin{align}
    \scalebox{0.96}{$\displaystyle\theta_{k}^\top \hat{v}_k = \frac{\ind\{i_k=\oo\}}{p_{k,\oo}} \left[p_{k,\xx}\theta_{k}(2)\cdot\tfrac{1}{2}T^{1/4}\hat{L}_k + p_{k,\oo}\theta_{k}(1)\left(T^{3/4}-\tfrac{1}{2}u\sqrt{T}\hat{L}_k\right) + p_{k,\oo}\theta_{k}(2)\sqrt{T}\right].$} \label{eq: res term}
\end{align}
We would like to find $(p_{k,\xx}, p_{k,\oo})$ so that \pref{eq: res term} can be controlled under the worst-case realization of $\hat{L}_k$. We simply choose $(p_{k,\xx}, p_{k,\oo})$ so that the $\Lhat_k$ terms on the right-hand side of \pref{eq: res term} cancel out:  
\begin{align}
    p_{k,\xx} \theta_{k}(2)\cdot\tfrac{1}{2} T^{1/4}- p_{k,\oo}\theta_{k}(1) \cdot \tfrac{1}{2}u \sqrt{T} = 0 \ \  \Rightarrow \ \  (p_{k,\xx}, p_{k,\oo}) =  \frac{\big(u\theta_{k}(1)\sqrt{T}, \ \ \theta_{k}(2)T^{1/4}\big)}{u\theta_{k}(1)\sqrt{T} + \theta_{k}(2)T^{1/4}}. \label{eq: p choice warm}
\end{align}
Plugging this back to \pref{eq: res term}, we get 
\begin{align}
    \theta_k^\top \hat{v}_k = \frac{\ind\{i_k=\oo\}}{p_{k,\oo}} \cdot \frac{\theta_{k}(2)T^{1/4}\big(\theta_{k}(1)T^{3/4} + \theta_{k}(2)\sqrt{T}\big)}{u\theta_{k}(1)\sqrt{T} + \theta_{k}(2)T^{1/4}} \leq \frac{\ind\{i_k=\oo\}}{p_{k,\oo}} \left(\frac{\sqrt{T}}{u} + \sqrt{T}\right), \label{eq: residual}
\end{align}
where we use that $\theta_k(2)\leq 1$. This gives $\E[\sum_{k=1}^{K} \theta_k^\top \hat{v}_k ] \leq O(T)$, as desired.  


\begin{algorithm}[t]
\caption{Warm-Up}
\label{alg: restricted}
\nl \textbf{Parameters}: $\eta=\frac{1}{10T}$, $\psi(\theta)=\sum_{a=1}^2 \log(1/\theta(a))$.  \\ 
\nl Initialize $\theta_1=(\frac{1}{2}, \frac{1}{2})$.  \\
\nl \For{$k=1, \ldots, \sqrt{T}$}{
\nl     Let $(p_{k,\xx}, p_{k,\oo})$ be defined as \pref{eq: p choice warm} with $u=1/3$. \\
\nl     Sample $i_k\in\{\xx,\oo\}$ according to distribution $(p_{k,\xx}, p_{k,\oo})$.  \\
\nl     Execute strategy $i_k$ in epoch $k$ (specified in \pref{eq: def of o and x}), and obtain $\Lhat_k$ if $i_k=\oo$.  \\
\nl     Construct $\hat{\ell}_{k,\xx}, \hat{\ell}_{k,\oo}$ as in \pref{eq: warmup estimator}.  \\
\nl     Let $\hat{v}_k = p_{k,\xx}\hat{\ell}_{k,\xx} + p_{k,\oo}\hat{\ell}_{k,\oo}$ and perform \broad update:  \\
\nl     $\theta_{k+1} =\argmin_{\theta\in\Delta_2}\left\{  \eta\inner{\theta, -\hat{v}_k + 2\eta \theta_k\hat{v}_k^2} + D_\psi(\theta, \theta_k)\right\}$, 
    where $(\theta_k\hat{v}_k^2)(i) \ldef\theta_{k}(i)\hat{v}_{k}(i)^2$. 
}
\end{algorithm}

\subsection{Online Learning Procedure for $\theta_k$ and Bound on $\calR(a)$}\label{sec: online learning proce} Ideally, we would like to achieve $\sum_{k=1}^K
(\hat v_k(a) - \theta_k^\top \hat v_k) \leq \calR(a)=O(T)$. However, we find that the online learning problem for $\theta_k$ is non-standard,
making direct application of standard bandit regret bounds insufficient. Let us elaborate this below. 
Recall that the problem has $K=\sqrt{T}$ epochs, and the target regret bound $\calR(a)$ is $O(T)$. Since regret bounds for multi-armed
bandits with $K$ rounds typically scale as $|\text{payoff scale}|\sqrt{K}$, this would require the (pre-importance-weighted) payoff scale in
each epoch to be $O(T^{3/4})$. However, from \pref{eq: warmup estimator}, the negative component of the payoff
can scale as $\sqrt{T}\,\hat L_k$, which may be as large as $\Theta(T)$. This prevents us from getting the desired bound. 

Fortunately, the specific structure of our problem allows us to relax the
requirement on $\calR(a)$. In fact, it suffices to prove the following bound: 
\begin{align}    \textstyle \sum_{k=1}^{K}\hat{v}_{k}(a) \leq O(T) + \frac{1}{2}\sum_{k=1}^{K}\left|\hat{v}_{k}(a)\right| + \sum_{k=1}^{K} \theta_k^\top \hat{v}_k.   \label{eq: surrogate}
\end{align}
In other words, it is enough to ensure $\calR(a)=O(T) + \frac{1}{2}\sum_{k=1}^{K}\left|\hat{v}_{k}(a)\right|$. To see why this suffices, consider first the case $a=1$ (corresponding to static regret). We have 
\begin{align}
    \hat{v}_{k}(1) - \tfrac{1}{2}|\hat{v}_{k}(1)| 
    &= p_{k,\oo}\big( \hat{\ell}_{k,\oo}(1) - \tfrac{1}{2}|\hat{\ell}_{k,\oo}(1)| \big) \tag{$p_{k,\xx} \hat{\ell}_{k,\xx}(1)=0$ by \pref{eq: warmup estimator}} \\
    &= \ind\{i_k=\oo\}\big(T^{3/4} - \tfrac{1}{2}u\sqrt{T}\hat{L}_k\big) - \tfrac{1}{2}\ind\{i_k=\oo\}\big|T^{3/4} - \tfrac{1}{2}u\sqrt{T}\hat{L}_k\big| \tag{by \pref{eq: warmup estimator}}\\
    &\geq \tfrac{1}{2}\ind\{i_k=\oo\}\big(T^{3/4} - \tfrac{3}{2}u\sqrt{T}\hat{L}_k\big)  \label{eq: include u}
\end{align}
where the last inequality follows by a case discussion on the sign of $T^{3/4} - \tfrac{1}{2}u\sqrt{T}\hat{L}_k$. Recall from \pref{eq: Sreg bound} and \pref{eq: Dreg bound} that the static regret satisfies $\SReg_{k,\oo}\leq \ind\{i_k=\oo\}(T^{1/4}-\tfrac{1}{2}L_k) \leq \ind\{i_k=\oo\}(T^{1/4}-\tfrac{1}{2}\Lhat_k)$, which is upper bounded by $\frac{2}{\sqrt{T}} (\hat{v}_{k}(1) - \frac{1}{2}|\hat{v}_{k}(1)|)$ if we choose $u= \frac{1}{3}$ in \pref{eq: include u} (in fact, any $u\in(0,\frac{1}{3}]$ works). This can be further upper bounded with \pref{eq: surrogate} as $\frac{2}{\sqrt{T}}\sum_k \E[\hat{v}_{k}(1) - \frac{1}{2}|\hat{v}_{k}(1)|]\leq \frac{2}{\sqrt{T}}\big(O(T)+\sum_{k} \E[\theta_k^\top \hat{v}_k ]\big)\leq O(\sqrt{T})$. For $a=2$ (corresponding to dynamic regret), the argument is simpler: since $\hat{v}_{k}(2)$ is non-negative, \pref{eq: surrogate} implies $\E[\sum_k \hat{v}_{k}(2)]\leq 2 (O(T)+\sum_{k} \E[\theta_k^\top \hat{v}_k ])\leq O(T)$, giving $\DReg\leq O(T^{3/4})$.  

Finally, we note that the desired bound \pref{eq: surrogate} is achievable.
In particular, it can be achieved with \broad, an online mirror-descent algorithm with a log-barrier
regularizer and second-order correction terms \citep{wei2018more}. We provide its guarantees in \pref{app: broad}.


\vspace*{5pt}
\paragraph{Combining everything}

Combining the elements developed in
\pref{sec: loss construction warmup}--\pref{sec: online learning proce},
we present the complete algorithm for the warm-up example in
\pref{alg: restricted}. The following theorem states its regret guarantee; the
proof is deferred to \pref{app: proof main thm for warmup}.
\begin{theorem}\label{thm: warmup thm}
For the warm-up example, \pref{alg: restricted} simultaneously achieves
$\E[\SReg] = O(\sqrt{T})$ and $\E[\DReg] = O(T^{3/4})$.
\end{theorem}

\section{General Algorithm}\label{sec: full algo}
We now consider general deterministic bandits with $A$ arms.
The loss vector $c_t \in [0,1]^A$ is chosen by an adversary prior to any
interaction with the learner and satisfies
$\sum_{t=2}^T \ind\{c_t \neq c_{t-1}\} = S-1$, where $S$ is known to the learner.
Our algorithm for this setting is given in \pref{alg: general}.

Unlike the warm-up example in \pref{sec: warmup}, we no longer assume that arm~1 is optimal. As a result, there are $A+1$ regret quantities
to control simultaneously: the static regret with respect to each arm,
$\SReg(1),\ldots,\SReg(A)$, and the dynamic regret $\DReg$.
We divide the horizon into $\sqrt{T/A}$ epochs, each of length $\sqrt{AT}$.
Let $\calI_k = [e_k, t_k]$ denote the $k$-th epoch.
Then we define
\begin{align}
   \textstyle L_k(a)
   =
   \sum_{t \in \calI_k}
   \left(
   c_t(a) - \min_{a' \in [A]} c_t(a')
   \right), \qquad s_k = \sum_{t=e_k+1}^{t_k} \ind\{c_t\neq c_{t-1}\}.  
   \label{eq: L_k definitions}
\end{align}
The quantity $L_k(a)$ plays the same role as $L_k$ in the warm-up example, but is now
defined through loss gaps rather than sub-interval lengths.
As in the warm-up, $L_k(a)$ captures the difference between the dynamic regret and the static regret against
arm~$a$, and is crucial in quantifying the negative static regret. The algorithm also estimates $s_k$, the number of loss switches in epoch $k$. This was ignored in the warm-up example, as $s_k=1$ always holds there. In each epoch, the learner commits to one of the two strategies:
$\xx$, which emphasizes static regret control, or $\oo$, which emphasizes
dynamic regret control.
A meta-algorithm based on Blackwell approachability hedges between these choices
across epochs. 
Below we describe the two strategies.



\begin{algorithm}[t]
\caption{\algname(\alg)}
\label{alg: general}
\nl \textbf{Parameters}: $\delta\in(0,1)$, $\eta=\frac{1}{20T}$, $\io=\log(\frac{AT}{\delta})$,  $\gamma=\frac{4\io}{\sqrt{AT}}$, $\rho=4\io\sqrt{\frac{SA}{T}}$, $u=\frac{9}{16}$ (assume $\rho\leq \frac{1}{2}$ without loss of generality). \\
\nl \textbf{Feasible set and regularizer}: Define $\Theta = \big\{\theta\in\Delta_{A+1}:~ 1-\theta(A+1) \geq A\gamma\big\}$ and   \label{line: clipping definition feas} \\[-10pt]
\begin{align*}
 \psi(\theta)=\sum_{a=1}^{A+1} \theta(a)\ln\theta(a) + \ln\frac{1}{\theta(A+1)} + \ln\frac{1}{1-\theta(A+1)}. 
\end{align*}   \label{line: feasiblet set} \\[-7pt]
\nl \textbf{Initialization}: $\theta_1\gets \argmin_{\theta\in\Theta} \psi(\theta)$.\\
\nl \For{$k=1, \ldots, \sqrt{T/A}$}{
\nl     Define $e_k=(k-1)\sqrt{AT}+1$ and $t_k=k\sqrt{AT}$. Let \\[-10pt]
    \begin{align}
         p_{k,\oo} = \max\left\{\frac{\theta_{e_k}(A+1)}{{\color{red}u}\sqrt{S}\sum_{a\in[A]} \theta_{e_k}(a) + \theta_{e_k}(A+1)}, \frac{1}{\sqrt{S}}\right\}, \qquad p_{k,\xx} = 1-p_{k,\oo}.    \label{eq: pkopox} 
    \end{align} \\[-7pt]
\nl     Sample $i_k\in\{\xx,\oo\}$ from the distribution  $(p_{k,\xx}, p_{k,\oo})$.\label{line: pick x or o}\\
\nl     \eIf{$i_k=\xx$}{
\nl         \For{$t=e_k, \ldots, t_k$}{
\nl             Choose $a_t\sim q_t$ and receive $c_{t}(a_t)$, where $q_t(a) \ldef \frac{\theta_t(a)}{ \sum_{a'\in[A]}  \theta_t(a')}$. \label{line: q_t def}\\
\nl             $\hat{v}_t(a) \ldef \sqrt{\frac{T}{A}}(c_t(a_t) - \hat{c}_t(a))$ where $\hat{c}_t(a) = \frac{\ind\{a_t=a\}c_t(a)}{q_{t}(a) + \gamma}$ for $a\in[A]$. \label{line: vt def in x}\\
\nl             $\hat{v}_t(A+1)\ldef 0$. \\
\nl             $\theta_{t+1} \gets \argmin_{\theta\in\Theta}\left\{  \eta\inner{\theta, -\hat{v}_t} + D_\psi(\theta, \theta_t)\right\}$. \label{line: update omd in x}
        }
    }{
\nl         \mbox{Run \expalg (\pref{alg: auer et l - det}) with exploration rate $\exprate$ for $t\in [e_k, t_k]$. Obtain $(\hat{L}_{k}(a))_{a=1}^A$ and $\hat{s}_k$. \label{line: proceed expalg}}\\
\nl         \For{$t=e_k, \ldots, t_k$}{
\nl             \If{$t=t_k$\label{line: start construct loss est}}{
\nl                 $\hat{v}_t(a) \ldef \min\left\{8\rho T + \frac{12\io\sqrt{AT}\hat{s}_k}{\rho}, \frac{3}{2}T \right\} - {\color{red}u}\sqrt{\frac{T}{A}}\cdot\Lhat_k(a)$  for $a\in[A]$. \label{line: vt in x}\\
\nl                 $ \hat{v}_{t}(A+1) \ldef \frac{p_{k,\xx}}{ p_{k,\oo}}\sqrt{\frac{T}{SA}} \cdot \sum_{a\in [A]} q_{e_k}(a) \Lhat_{k}(a)$.
            }
\nl         \lElse{
                 $\hat{v}_t(a)\ldef 0$ for $a\in[A+1]$. \label{line: vt = 0 for all t}
            }\label{line: end construct loss est}
\nl             $\theta_{t+1} \gets \argmin_{\theta\in\Theta}\left\{  \eta\inner{\theta, -\hat{v}_t + 2\eta \theta_t\hat{v}_t^2} + D_\psi(\theta, \theta_t) \right\}$. \label{line: broad-omd update1}
        }
    }
}
\end{algorithm}

\begin{algorithm}[t]
\caption{\expalgCName (\expalg) in epoch $k$ with $i_k=\oo$}
\label{alg: auer et l - det}
\nl \textbf{Input}: exploration parameter $\exprate$.\\
\nl Initialize $\hat{c}_{e_k}(a)\leftarrow 0$ for all $a\in[A]$ and $t\leftarrow e_k$.\\
\nl \For{$m=0,1,2, \ldots, $}{
\nl     $\tau\super{m}\leftarrow t$.   \label{line: starting index}\\
\nl     \command{Phase 1:~Sample each arm once} \\
\nl     \For{$i=1,2, \ldots, A$}{
\nl         \lIf{$t>t_k$}{\textbf{break}.} 
\nl         Pull arm $a_t = i$, and observe cost $c_t(a_t)$.\\
\nl         Update $\hat{c}_{t+1}(a_t) \leftarrow c_t(a_t)$, and keep $\hat{c}_{t+1}(a)\leftarrow \hat{c}_{t}(a)$ for all $a\neq a_t$.\\
\nl         $t\leftarrow t+1$.
    }
\nl $\Delta\super{m}(a) \gets \hat{c}_t(a) - \min_{a'\in[A]} \hat{c}_t(a')$. \label{line: calculate gap} \\
\nl  $\prob\super{m}(a) = \begin{cases}
    \frac{1}{A}\min\big\{1, \frac{\rho}{\Delta\super{m}(a)} \big\}, \qquad &a\neq \argmin_{a'\in[A]} \hat{c}_t(a'), \\
    1-\sum_{a'\neq a}\prob\super{m}(a'),   &a=\argmin_{a'\in[A]} \hat{c}_t(a'). \end{cases}$   \label{line: gap dep exp}\\[5pt] 
\nl     \command{Phase 2:~Sample arms according to $q\super{m}$ and detect changes}  \\
\nl     \While{$t\leq t_k$}{
\nl         Pull arm $a_t\sim \prob\super{m}$. \\
\nl         Update $\hat{c}_{t+1}(a_t) \leftarrow c_t(a_t)$, and keep $\hat{c}_{t+1}(a)\leftarrow \hat{c}_{t}(a)$ for all $a\neq a_t$.\\
\nl         \changedetected\ $\leftarrow \mathbb{I}\{c_t(a_t)\neq \hat{c}_{t}(a_t)\}$.\\
\nl         $t\leftarrow t+1$.\\
\nl         \lIf{\changedetected}{
            \textbf{break}. \label{line: change detected}
        }
    } 
\ \\
\nl   \lIf{$t>t_k$}{\textbf{break}.}
}
\nl \textbf{return}  $\Lhat_k(a) = \sum_{t=e_k}^{t_k} \left(\hat{c}_{t}(a) - \min_{a'\in [A]} \hat{c}_{t}(a')\right)$ for $a\in[A]$, and $\hat{s}_k = \text{final value of $m$}$.  \label{line: hatL definition}
\end{algorithm}

\paragraph{The $\xx$ strategy}
If $\xx$ is selected in epoch~$k$, then in each round the learner
samples an arm according to
$q_t(a) = \theta_t(a) / \sum_{a' \in [A]} \theta_t(a')$
(\pref{line: q_t def}), where
$\theta_t \in \Delta_{A+1}$ is the weight vector maintained by the
approachability algorithm.
Using the observed loss, the learner constructs standard bandit loss
estimators $\hat{c}_t(a)$ as in \cite{neu2015explore}, and uses them to construct the regret vector
$\hat{v}_t$ (\pref{line: vt def in x}). Unlike in the warm-up, $\vhat_t$ is indexed by $t$ here, and the approachability algorithm is updated every round. 
\paragraph{The $\oo$ strategy}
If $\oo$ is selected in epoch~$k$, the learner runs the procedure
\expalg\ (\pref{alg: auer et l - det}). \expalg is a detection-based dynamic regret minimization algorithm inspired by \cite{auer2019adaptively} and \cite{karnin2016multi}. 
At the end of the epoch, \expalg returns $\hat{L}_k(a)$ and 
$\hat{s}_k$ as estimates of $L_k(a)$ and $s_k$. 
These estimates are used to construct the regret vector
$\hat{v}_t$ (\pref{line: start construct loss est}--\pref{line: end construct loss est}).
In $\oo$ epochs, although $\hat{v}_t$ is defined for every round $t$, it is non-zero only at the final
round of the epoch. 


Below, our analysis follows the same high-level structure as in the warm-up case.
We establish static and dynamic regret guarantees
for the $\xx$ and $\oo$ strategies individually.
We then relate these guarantees to the components of the regret vector
$\hat{v}_t$ and analyze the approachability algorithm operating on $\vhat_t$.
All results are stated with high-probability guarantees.

\subsection{Approachability Algorithm Design and Analysis} 
\label{sec:base-procedure}

\paragraph{Guarantees in $\xx$ epochs} 
If $i_k=\xx$, we estimate the loss of each arm using the
\mbox{EXP3-IX} estimator $\hat{c}_t(a)$ \citep{neu2015explore}, and thus estimate the instantaneous static
regret against arm $a$ by $c_t(a_t)-\hat c_t(a)$. For dynamic regret, \pref{lem:one-interval-omd} shows that with high probability ($e_k$ is the first round of epoch $k$)
\begin{align}
\textstyle\DReg_{k,\xx} \leq \sum_{a\in[A]} q_{e_k}(a) L_k(a) + \sum_{a\in[A]}\sum_{t\in \calI_k} q_{e_k}(a)\left(\hat{c}_t(a) - c_t(a)\right) +  \frac{2.5A\io}{\sum_{a\in[A]} \theta_{e_k}(a)}.    \label{eq: Dreg x general}
\end{align}



\paragraph{Guarantees in $\oo$ epochs}
If $i_k=\oo$, we run \expalg in the epoch with exploration parameter $\exprate\approx\sqrt{SA/T}$. 
\pref{lemma: sreg under ooo} and \pref{lem: dreg in o} show that $\expalg$ guarantees with high probability
    \begin{align}
        \SReg_{k,\oo}(a) &\leq O\big( \min\big\{\exprate\sqrt{AT}+ \exprate^{-1}A\hat{s}_k\io,\ \sqrt{AT}\big\}\big) - \tfrac{3}{4}\Lhat_k(a), \label{eq: SRego in general} \\
        \DReg_{k,\oo} &\leq O\big(\min\big\{\exprate\sqrt{AT}+ \exprate^{-1}A\left(\hat{s}_k + \ind\{s_k>0\}\right)\io, \ \sqrt{AT}\big\}\big). \label{eq: DRego in general} 
    \end{align}
In our proof, it is crucial that all quantities in the static regret upper bound \pref{eq: SRego in general} are observable by the learner (in contrast to \pref{eq: DRego in general}, where $\ind\{s_k>0\}$ is not observable). The positive and negative terms in \pref{eq: SRego in general} capture the exploration overhead and the negative static regret, respectively, and have to be estimated with sufficient accuracy to ensure cancellation. To this end, it is important to employ the gap-dependent exploration in \pref{line: gap dep exp}. In fact, the sole purpose of assuming deterministic losses in this work is to enable the derivation of \pref{eq: SRego in general}---without deterministic losses, the additional detection cost would make the static regret in $\oo$ worse and lead to sub-optimal results finally. We discuss \expalg more in \pref{app: expalg}.

\paragraph{Forming the regret vector $\vhat_t$}
Similar to the warm-up case, the regret vector's components correspond to the regrets we aim to control. Over epoch $k$, we define $\hat{v}_t\in\mathbb{R}^{A+1}$
so that
\begin{align}
    \sum_{t\in\calI_k}\vhat_t = p_{k,\xx} \begin{bmatrix}
        \sqrt{T/A} \cdot \overline{\SReg}_{k,\xx}(1) \\
        \vdots \\
        \sqrt{T/A} \cdot \overline{\SReg}_{k,\xx}(A) \\
        \sqrt{T/(SA)}\cdot \overline{\DReg}_{k,\xx}
    \end{bmatrix} + p_{k,\oo} \begin{bmatrix}
        \sqrt{T/A} \cdot \overline{\SReg}_{k,\oo}(1) \\
        \vdots \\
        \sqrt{T/A} \cdot \overline{\SReg}_{k,\oo}(A) \\
        \sqrt{T/(SA)}\cdot \overline{\DReg}_{k,\oo}  
    \end{bmatrix}
    \label{eq: complicated estimator}
\end{align}
for some regret (upper bound) estimators $\overline{\SReg}_{k,\xx}(a)$, $\overline{\SReg}_{k,\oo}(a)$, $\overline{\DReg}_{k,\xx}$, and $\overline{\DReg}_{k,\oo}$ derived based on \pref{eq: Dreg x general}--\pref{eq: DRego in general}, with the scaling chosen so that it suffices to control all objectives to $O(T)$.  

Not all terms in the regret bounds \pref{eq: Dreg x general}--\pref{eq: DRego in general} need to be included in the definition of $\vhat_t$. 
In particular, terms that can be controlled for any choice of
$(p_{k,\xx},p_{k,\oo})$ do not need to be hedged with other regrets
and can therefore be omitted. This allows us to omit the last two terms in \pref{eq: Dreg x general} and all terms in \pref{eq: DRego in general}.
Accordingly, our $\vhat_t$ in \pref{alg: general} is constructed according to \pref{eq: complicated estimator} with 
\begin{align*}
    &\overline{\SReg}_{k,\xx}(a) = \textstyle\frac{\ind\{i_{k}=\xx\}}{p_{k,\xx}}\sum_{t\in\calI_k} (c_t(a_t) - \hat{c}_t(a)), \quad \overline{\DReg}_{k,\xx} = \frac{\ind\{i_k=\oo\}}{p_{k,\oo}} \sum_{a\in[A]} q_{e_k}(a)\Lhat_k(a), \notag\\
    &\overline{\SReg}_{k,\oo}(a) = \textstyle \frac{\ind\{i_k=\oo\}}{p_{k,\oo}}  \left(\min\left\{8\rho\sqrt{AT} + 12\rho^{-1}A\hat{s}_k\io, \frac{3}{2}\sqrt{AT}\right\} - \frac{9}{16}\Lhat_k(a)\right), \quad \overline{\DReg}_{k,\oo} = 0.  
\end{align*}
Besides, there are three \emph{clipping operations} in the design of \pref{alg: general}: the $\min\{\cdot~, \frac{3}{2}T\}$ in \pref{line: vt in x}, the constraint $p_{k,\oo}\geq \frac{1}{\sqrt{S}}$ in \pref{eq: pkopox}, and the constraint $1-\theta(A+1) = \sum_{a\in[A]}\theta(a)\geq A\gamma$ in \pref{line: clipping definition feas}. They are used to control the magnitudes of certain regret terms to facilitate high-probability guarantees.

\paragraph{Approachability guarantees}
\label{sec:meta-alg}
Define $Z=
\sum_{k=1}^K\sum_{t\in \calI_k} \left(\hat{v}_t \ind\{i_k=\xx\}  +  \left(\hat{v}_t - \frac{1}{3}|\hat{v}_t|\right) \ind\{i_k=\oo\}\right)\in\mathbb{R}^{A+1}$. \pref{lemma: connect sreg to v} and \pref{lem: connect dreg to v} establish the following bounds with high probability: 
\begin{align*}
   \forall a\in[A], \quad  \SReg(a) \textstyle\leq \sqrt{\frac{A}{T}}Z(a) + O(\sqrt{AT}\io) \quad \text{and} \quad
   \DReg \textstyle\leq 6\sqrt{\frac{SA}{T}}Z(A+1) + O(\sqrt{SAT}\io).   
\end{align*}
\pref{lem: final regret bound} further shows $
   Z(a) \textstyle\leq O(T\io)$ for all $ a\in[A+1]$. Combining them, we get: 
\begin{theorem}
     \pref{alg: general} ensures with probability at least $1-O(\delta)$ that $\SReg(a)\leq O\big(\sqrt{AT}\log(T/\delta)\big)$ for all $a\in[A]$ and $\DReg\leq O\big(\sqrt{SAT}\log(T/\delta)\big)$ simultaneously. 
\end{theorem}






\section{Conclusion}
We present the first algorithm achieving optimal static and dynamic regret simultaneously in adversarial multi-armed bandits, under the assumptions of deterministic losses and a known number of switches. Relaxing either assumption while maintaining optimality presents significant challenges and is left as future work. We hope our work provides new insights into the ``best-of-all-worlds'' problem about the existence of an algorithm that competes optimally against benchmarks with different numbers of switches simultaneously---an open problem stated in \cite{auer2019adaptively}. 

\section*{Acknowledgements}


We thank Zak Mhammedi for participating in the early stages of this project and providing insightful discussions.

\bibliographystyle{abbrvnat}  
\bibliography{arxiv/references}
\clearpage
\appendix

\newpage
\section{Related Work}\label{app: related work}

Our work bridges static regret minimization for adversarial bandits and dynamic regret minimization for non-stationary bandits, each of which has been studied for decades. We refer the reader to \cite{bubeck2012regret} and \cite{lattimore2020bandit} for a comprehensive survey on each of them. 

The EXP3.S algorithm by \cite{auer2002nonstochastic} is a powerful framework that can achieve $\otil(\sqrt{\Gamma T})$ regret if the learning rate is tuned perfectly with respect to $\Gamma$, where $\Gamma-1$ is the number of times the benchmark sequence switches the arm. When $\Gamma$ is configured as $1$, it recovers the $\otil(\sqrt{T})$ static regret; when configured as the number of times the loss changes, it recovers the $\otil(\sqrt{ST})$ dynamic regret.  Unfortunately, these two bounds are achieved under different choices of learning rates, and there is no obvious way to achieve them simultaneously. 

Recently, there is a line of research focusing on achieving optimal dynamic regret without prior knowledge on the number of times the loss changes or the total variation of the losses. The first such result is by \cite{auer2019adaptively}, which is further extended or refined by \cite{chen2019new, wei2021non, chen2021combinatorial, chen2022near, suk2022tracking, suk2023tracking, abbasi2023new,  ganguly2023online, hong2023optimization, wang2025adaptivity}. The focus of these works are different from ours: they solely focuses on dynamic regret (and may sacrifice static regret), while we simultaneously care about static regret and dynamic regret. On the other hand, their algorithms are prior-knowledge-free and can handle stochastic losses, while our current algorithm relies on prior knowledge of $S$ and works only for deterministic losses.  

Our problem can be regarded as a model selection problem. In model selection, the learner is faced with a set of plausible models of the world or a set of base algorithms, and the goal is to perform seamlessly as well as the algorithm that knows the true model in advance or the best base algorithm. In many bandit settings, due to limited feedback, it has been shown to be impossible to approach the performance of the best base algorithm through model selection \citep{pacchiano2020model, marinov2021pareto, zhu2022pareto}. The $S$-switch bandit problem studied in our work was extensively discussed in \cite{marinov2021pareto}. While they focus on the lower bound for adaptive adversary, we focus on the upper bound for oblivious adversary. 

In bandits, there are a few results formally showing a gap between the power of adaptive and oblivious adversary. One notable example is given by \cite{auer2016algorithm}, who showed that in multi-armed bandits, when the adversary is adaptive, it is impossible to achieve the \emph{best of both worlds}, i.e., achieving an $O(\log T)$ regret when the losses are i.i.d., and $\otil(\sqrt{T})$ when the losses are adversarial. However, this is possible for oblivious adversary \citep{auer2016algorithm, zimmert2021tsallis}. As another example,  \cite{bubeck2019improved} showed a gap between the two adversaries regarding path-length regret bounds.  More recently, \cite{filmus2024bandit} showed the strictly stronger power of adaptive adversary in the multi-class bandit classification problem.  

Blackwell approachability \citep{blackwell1956analog} was proposed as a generalization to von Neumann's minimax theorem \citep{v1928theorie} for vector-valued payoff. It can be naturally applied to online learning problems with multiple objectives, such as calibrated forecasting \citep{foster1999proof}, fair learning \citep{chzhen2021unified}, and scheduling \citep{hou2009admission}. It has been shown that Blackwell approachability and no-regret learning are equivalent in the sense that an algorithm for one problem can be converted to an algorithm for the other \citep{abernethy2011blackwell, perchet2013approachability, perchet2015exponential}. In this work, we also perform approachability by leveraging no-regret learning algorithms.

\section{\pfref{thm: lower bound general S}}
\label{app: lower bound proof}

\begin{proof}[\pfref{thm: lower bound general S}]
     Without loss of generality, we assume that $S/2$ and $2T/S$ are integers. The strategy of the adversary is as follows. Divide the horizon into $S/2$ epochs each of length $2T/S$. The default loss vector at the beginning of each epoch is $c_t=(\frac{1}{2}, 1)$. Let $M$ and $N$ be two integers to be determined later. Conditioned on the adversary's decisions up to epoch $(k-1)$, the adversary makes decisions in epoch $k$ based on the following: 
    \begin{itemize}
         \item \textbf{Case 1}: If the learner's expected number of draws to arm $2$ in epoch $k$ is at least $\PHtwo$, the adversary will keep $c_{t}=(\frac{1}{2},1)$ throughout the whole epoch $k$. 
         \item \textbf{Case 2}: If the expected number of draws to arm $2$ in epoch $k$ is smaller than $\PHtwo$, the adversary will find a \emph{sub-interval} of length $\PHthree$ within epoch $k$, and change the loss to $c_{t}=(\frac{1}{2}, 0)$ in that sub-interval, while keeping $c_{t}=(\frac{1}{2}, 1)$ outside that sub-interval. This sub-interval is chosen so that the expected number of draws to arm $2$ is minimized among all windows of length $\PHthree$. 
    \end{itemize}
    Up to this point, the adversary is oblivious, since all decisions above only requires knowledge on the learner's algorithm, but not the real interaction history between the learner and the adversary. The adversary will further do the following that makes it indeed adaptive: in the Case 2 above, if the learner ever draws arm 2 in that sub-interval (i.e., the learner finds that the loss has been changed!), the adversary immediately change the loss back to $c_{t}=(\frac{1}{2}, 1)$ in the remaining rounds in that epoch. It is easy to see that the loss vector switches at most $2\cdot S/2=S$ times. 

    Now we analyze the expected regret of the learner. We first calculate $\SReg$. For each epoch, if Case 1 holds, then the learner's regret against arm~1 adds at least $\frac{\PHtwo}{2}$; if Case 2 holds, the learner has at least $-\frac{1}{2}$ regret against arm~1 because the learner wins over arm $1$ by an amount of $\frac{1}{2}$ in at most one round (recall that if this happens, the adversary soon changes the loss vector back). Overall, 
    \begin{align}
        \SReg \geq\SReg_1 \geq \E\left[\sum_{k=1}^{S/2} 
 \left(E_{k}\times \frac{\PHtwo}{2} - (1-E_k)\times\frac{1}{2}\right)\right],   \label{eq: sreg lower general}
    \end{align}
    where $E_k=1$ indicates that Case 1 holds for epoch $k$, and $0$ otherwise. 

    Next, we calculate $\DReg$. The benchmark will be arm~1 if $c_t=(\frac{1}{2},1)$ and arm~2 if $c_t=(\frac{1}{2}, 0)$. If Case~1 holds in epoch $k$, then the dynamic regret adds at least $\frac{\PHtwo}{2}$. If Case~2 holds and if the learner never samples arm $2$ in the sub-interval, then the dynamic regret adds at least $\frac{1}{2}\times\PHthree=\frac{\PHthree}{2}$. By an argument provided at the end of the proof, if $MN\leq T/S$, then we can show that in Case~2 this happens with probability at least $\frac{1}{2}$. 
    Overall, we have
    \begin{align}
        \DReg \geq \E\left[\sum_{k=1}^{S/2} \left(E_k\times \frac{\PHtwo}{2} + (1-E_k)\times\frac{1}{2}\times  \frac{\PHthree}{2}\right)\right].   \label{eq: dreg lower general}
    \end{align}
    With \pref{eq: sreg lower general} and \pref{eq: dreg lower general}, we see that if $\E[\sum_k E_k] \geq \frac{3}{4} \times \frac{S}{2}$, then $\SReg \geq \frac{S}{2} \times \frac{\PHtwo}{2} = \frac{S\PHtwo }{4}$; otherwise, $\DReg\geq \frac{S}{8}\times \frac{\PHthree}{4}= \frac{S\PHthree}{32}$. This proves the theorem with the choice of $\PHtwo = \sqrt{T}/S^{1-\alpha}$ and $\PHthree =  \sqrt{T}/S^\alpha$.

    Finally, we show that in any epoch, if Case~2 holds, then with probability at least $\frac{1}{2}$, the learner never samples arm~2 in the sub-interval. We prove it by contradiction. By the way the adversary chooses the sub-interval, if the statement does not hold, then for any window of length $\PHthree$, with probability at least $\frac{1}{2}$, the learner samples arm~2 at least once. Since there are $\frac{2T/S}{\PHthree}\geq 2\PHtwo$ such windows in epoch $k$, the expected number of samples of arm~2 is at least $2\PHtwo\times \frac{1}{2}=\PHtwo$, contradicting that we are in Case 2. 
\end{proof}



\section{Concentration Inequalities}
\begin{lemma}[Freedman's inequality (Lemma~A.3 of \cite{foster2021statistical})]\label{lem: freedman+}
    Let $X_1, \ldots, X_T$ be a sequence of random variable adapted to a filtration $(\calF_t)_{t=0}^T$. If $0\leq X_i\leq B$ almost surely, then with probability at least $1-\delta$, 
    \begin{align*}
        \sum_{t=1}^T X_t &\leq \frac{3}{2}\sum_{t=1}^T \E[X_t|\calF_{t-1}] + 4B\log(1/\delta). 
    \end{align*}
    Also, with probability at least $1-\delta$, 
    \begin{align*}
        \sum_{t=1}^T \E[X_t|\calF_{t-1}] &\leq 2\sum_{t=1}^T X_t + 8B\log(1/\delta).
    \end{align*}
\end{lemma}

\begin{lemma}[Lemma~1 of \cite{neu2015explore}]\label{lem: IX lemma}
Fix a deterministic sequence
$\ell_1,\ldots,\ell_T\in[0,1]^A$ and $\gamma>0$. 
Let $\mathbb P$ be a probability measure and let $(\mathcal F_t)_{t=0}^T$
be a filtration. 
For each $t\in[T]$, let $q_t\in\Delta_A$ and $\alpha_t\in[0,1]^A$ be
$\mathcal F_{t-1}$-measurable, and let $a_t$ be an $[A]$-valued random
variable such that
\[
\mathbb P(a_t=a \mid \mathcal F_{t-1}) = q_t(a)\qquad \text{a.s. for all } a\in[A].
\]
Define for $a\in[A]$,
\[
\hat{\ell}_t(a):=\frac{\ell_t(a)\mathbf \ind\{a_t=a\}}{q_t(a)+\gamma}.
\]
Then with probability at least $1-\delta$,
\[
\sum_{t=1}^T\sum_{a=1}^A \alpha_t(a)\bigl(\hat{\ell}_t(a)-\ell_t(a)\bigr)
\le \frac{\log(1/\delta)}{2\gamma}.
\]
\end{lemma}

\section{Mirror Descent} \label{app: broad}
The regret bound analysis for online mirror descent in this section is rather standard. We provide detailed proofs for completeness. 
\begin{lemma}[Lemma~20 of \cite{dann2023best}]\label{lem: dann1}
    Let $\psi(\theta)=\sum_{a=1}^A \theta(a)\ln\theta(a)$ and let $\ell_t\in\mathbb{R}^A$ be such that $\eta\ell_t(a)\geq -1$. Then for any convex subset $\Theta\subseteq \Delta_A$, 
    \begin{align*}
        \max_{\theta\in\Theta} \left\{ \langle \theta_t - \theta, \ell_t\rangle - \frac{1}{\eta}D_{\psi}(\theta, \theta_t)\right\} \leq \eta \sum_{a=1}^A \theta_t(a)\ell_t(a)^2. 
    \end{align*}
\end{lemma}

\begin{lemma}[Lemma~21 of \cite{dann2023best}]\label{lem: dann2}
    Let $\psi(\theta)=\sum_{a=1}^A \ln\frac{1}{\theta(a)}   $ and let $\ell_t\in\mathbb{R}^A$ be such that $\eta\theta_t(a)\ell_t(a)\geq -\frac{1}{2}$. Then for any convex subset $\Theta\subseteq \Delta_A$, 
    \begin{align*}
        \max_{\theta\in\Theta} \left\{ \langle \theta_t - \theta, \ell_t\rangle - \frac{1}{\eta}D_{\psi}(\theta, \theta_t)\right\} \leq \eta \sum_{a=1}^A \theta_t(a)^2\ell_t(a)^2. 
    \end{align*}
\end{lemma}

\subsection{Warm-up Case (for \pref{alg: restricted})} 
\begin{lemma}\label{lem: log-bar omd}
    Let $\Theta\subseteq \Delta_A$ be convex. Suppose $\eta \theta_{t}(a)|\hat{v}_{t}(a)|\leq \frac{1}{6}$ for all $t, a$, and let $\theta_t\hat{v}_t^2$ denote the vector $(\theta_{t}(1)\hat{v}_{t}(1)^2, \ldots, \theta_{t}(A)\hat{v}_{t}(A)^2)$. Then the online mirror descent update
    \begin{align*}
        \theta_{t+1} = \argmin_{\theta\in \Theta}\left\{ \inner{\theta, -\hat{v}_t + 2\eta\theta_t\hat{v}_t^2} + \frac{D_\psi(\theta, \theta_t)}{\eta}\right\}
    \end{align*}
    with log-barrier regularizer $\psi(\theta)=\sum_{a=1}^A \ln\frac{1}{\theta(a)}$ ensures  
    \begin{align*}
        \inner{\theta, \hat{v}_{t}} \leq \frac{D_\psi(\theta, \theta_t) - D_\psi(\theta, \theta_{t+1})}{\eta}  + 2\eta \inner{\theta, \theta_t \hat{v}_{t}^2} + \inner{\theta_{t}, \hat{v}_t} 
    \end{align*}
    for any $\theta\in\Theta$. 
\end{lemma}

\begin{proof}
    By the optimality of $\theta_{t+1}$, we have 
    \begin{align*}
         \inner{\theta_{t+1}, -\hat{v}_t + 2\eta \theta_t\hat{v}_t^2} + \frac{D_\psi(\theta_{t+1}, \theta_t)}{\eta} \leq \inner{\theta, -\hat{v}_t + 2\eta \theta_t\hat{v}_t^2} + \frac{D_\psi(\theta, \theta_t)}{\eta} - \frac{D_\psi(\theta, \theta_{t+1})}{\eta}. 
    \end{align*}
    Rearranging this gives 
    \begin{align}
         \inner{\theta, \hat{v}_t}
         &\leq \frac{D_\psi(\theta, \theta_t) - D_\psi(\theta, \theta_{t+1})}{\eta} + 2\eta \inner{\theta, \theta_t \hat{v}_{t}^2}  + \inner{\theta_t, \hat{v}_t} \nonumber \\
         &\qquad + \inner{\theta_{t} - \theta_{t+1}, - \hat{v}_t + 2\eta \theta_t\hat{v}_t^2} - \frac{D_\psi(\theta_{t+1}, \theta_t)}{\eta} - 2\eta \inner{\theta_t, \theta_t\hat{v}_t^2}.    \label{eq: omd tmp11}
    \end{align}
    By the assumption, it holds that
    \begin{align*}
        \eta \theta_{t}(a)\left| -\hat{v}_{t}(a) + 2\eta \theta_{t}(a)\hat{v}_{t}(a)^2\right| \leq \frac{1}{6} + 2\left(\frac{1}{6}\right)^2 < \frac{1}{2}.  
    \end{align*}
    By \pref{lem: dann2}, we get 
    \begin{align*}
        \inner{\theta_{t} - \theta_{t+1}, - \hat{v}_t + 2\eta \theta_t\hat{v}_t^2} - \frac{D_\psi(\theta_{t+1}, \theta_t)}{\eta} 
        &\leq \eta \sum_{a=1}^A \theta_{t}(a)^2 (-\hat{v}_{t}(a)+2\eta \theta_{t}(a)\hat{v}_{t}(a)^2)^2 \\
        &\leq \eta \sum_{a=1}^A \theta_{t}(a)^2 \left(\frac{4}{3}\hat{v}_{t}(a)\right)^2 \leq 2\eta \sum_{a=1}^A \theta_{t}(a)^2 \hat{v}_{t}(a)^2. 
    \end{align*}
    Combining this with \pref{eq: omd tmp11} finishes the proof. 
\end{proof}

\begin{lemma}
    \label{lem: log-bar omd smooth}
    Under the same conditions and the same update rules specified in \pref{lem: log-bar omd}, it holds that 
    \begin{align*}
        \sum_{t=1}^T \hat{v}_{t}(a) \leq 
         \frac{A}{\eta} \ln \frac{1}{\alpha}   + 2\eta \sum_{t=1}^T \theta_{t}(a) \hat{v}_{t}(a)^2 + \sum_{t=1}^T \inner{\theta_{t}, \hat{v}_t} +  \alpha\sum_{t=1}^T \left\langle e_a - \frac{\mathbf{1}}{A}, \hat{v}_t - 2\eta\theta_t \hat{v}_t^2 \right\rangle. 
    \end{align*}
    for any $a\in [A]$, and $\alpha\in(0,1]$. 
\end{lemma}

\begin{proof}
    Summing the bound in \pref{lem: log-bar omd} for $\theta'$ and $t\in [T]$, we get 
     \begin{align*}
        \sum_{t=1}^T \inner{\theta', \hat{v}_{t}} \leq \frac{D_\psi(\theta', \theta_{1})}{\eta}  + 2\eta \sum_{t=1}^T \inner{\theta', \theta_t \hat{v}_{t}^2} + \sum_{t=1}^T \inner{\theta_{t}, \hat{v}_t} .
    \end{align*}
    Choosing
    \begin{align*}
         \theta' = (1-\alpha)\e_a + \frac{\alpha}{A}\mathbf{1}, 
    \end{align*}
    the above becomes 
    \begin{align*}
        \sum_{t=1}^T \hat{v}_{t}(a) 
        &\leq \frac{D_\psi(\theta', \theta_{1})}{\eta}  + 2\eta \sum_{t=1}^T \theta_{t}(a) \hat{v}_{t}(a)^2 + \sum_{t=1}^T \inner{\theta_{t}, \hat{v}_t} +  \alpha\sum_{t=1}^T \left\langle e_a - \frac{\mathbf{1}}{A}, \hat{v}_t - 2\eta\theta_t \hat{v}_t^2 \right\rangle. 
    \end{align*}
    To complete the proof, notice that $\theta_1 = \frac{\mathbf{1}}{A}$ and 
    \begin{align*}
        D_\psi\left((1-\alpha)e_a + \frac{\alpha \mathbf{1}}{A}, \frac{\mathbf{1}}{A}\right) \leq A\ln \frac{1}{\alpha}. 
    \end{align*}
\end{proof}

\subsection{General Case (for \pref{alg: general})}
In this subsection, we consider the hybrid regularizer 
\begin{align}
        \psi(\theta) =  \sum_{a=1}^{A+1} \theta(a) \ln \theta(a) + \ln\frac{1}{\theta(A+1)} + \ln\frac{1}{1-\theta(A+1)}   \label{eq: hybrid}
\end{align}
over the convex feasible set $\Theta \subseteq \Delta_{A+1}$. 
\begin{lemma}[Hybrid Regularizer]\label{lem: hybrid stabiltiy}
    Let $\psi$ be defined as \pref{eq: hybrid}. 
    Then for any $\ell_t\in\mathbb{R}^{A+1}$ satisfying 
    \begin{align*}
        \eta\ell_t(a)\geq -1 \ \ \ \forall a\in [A]\qquad  \text{and}\qquad \eta \theta(A+1)\ell_t(A+1) \geq -\frac{1}{2}, 
    \end{align*}
    it holds that 
    \begin{align*}
         \max_{\theta\in \Theta} \left\{ 
 \left\langle  \theta_t - \theta, \ell_t  \right\rangle  - \frac{1}{\eta} D_\psi(\theta, \theta_t)\right\} \leq \eta \left(\sum_{a=1}^A  \theta_t(a) \ell_t(a)^2 + \theta_t(A+1)^2\ell_t(A+1)^2\right). 
    \end{align*}
\end{lemma}

\begin{proof}
    Define 
    \begin{align*}
    \psi_\NE(\theta) &= \sum_{a=1}^{A+1} \theta(a) \ln\theta(a), \\
        \psi_\LB(\theta) &= \log\frac{1}{\theta(A+1)} + \log\frac{1}{1-\theta(A+1)}. 
    \end{align*}
    Then we have 
    \begin{align*}
        &\max_{\theta\in \Theta} \left\{ 
 \left\langle  \theta_t - \theta, \ell_t  \right\rangle  - \frac{1}{\eta} D_\psi(\theta, \theta_t)\right\} \\
 &= \max_{\theta\in \Theta} \left\{ 
  \left\langle  \theta_t - \theta, \begin{bmatrix} \ell_t(1)\\ \vdots \\ \ell_t(A) \\ 0 \end{bmatrix}  \right\rangle + \left\langle  \theta_t - \theta, \begin{bmatrix} 0\\ \vdots \\ 0 \\ \ell_t(A+1) \end{bmatrix}  \right\rangle  - \frac{1}{\eta} D_{\psi_\NE}(\theta, \theta_t) - \frac{1}{\eta} D_{\psi_\LB}(\theta, \theta_t) \right\} \\
 &\leq \max_{\theta\in \Theta} \left\{ 
 \left\langle  \theta_t - \theta, \begin{bmatrix} \ell_t(1)\\ \vdots \\ \ell_t(A) \\ 0 \end{bmatrix}  \right\rangle  - \frac{1}{\eta} D_{\psi_\NE}(\theta, \theta_t) \right\} +  \max_{\theta\in \Theta} \left\{ 
 \left\langle  \theta_t - \theta, \begin{bmatrix} 0\\ \vdots \\ 0 \\ \ell_t(A+1) \end{bmatrix}  \right\rangle - \frac{1}{\eta} D_{\psi_\LB}(\theta, \theta_t)  \right\}\\
 &\leq \eta \sum_{a=1}^A \theta_t(a)\ell_t(a)^2 + \eta \theta_t(A+1)^2 \ell_t(A+1)^2,   
    \end{align*}
    where in the last inequality we use \pref{lem: dann1} and \pref{lem: dann2}. 
\end{proof}

\begin{lemma} \label{lem: lem without correction}
    Let $\Theta\subseteq \Delta_{A+1}$ be convex. Suppose that $\eta \hat{v}_t(a) \leq 1$ for all $a\in[A]$, and $\eta\theta_t(A+1) \hat{v}_t(A+1) \leq \frac{1}{2}$. Then the online mirror descent update
    \begin{align*}
        \theta_{t+1} = \argmin_{\theta\in \Theta}\left\{ \inner{\theta, -\hat{v}_t} + \frac{D_\psi(\theta, \theta_t)}{\eta}\right\}
    \end{align*}
    with hybrid regularizer $\psi$ in \pref{eq: hybrid} ensures for any $\theta\in\Theta$, 
    \begin{align*}
        \inner{\theta, \hat{v}_{t}} \leq \frac{D_\psi(\theta, \theta_t) - D_\psi(\theta, \theta_{t+1})}{\eta}  + \inner{\theta_{t}, \hat{v}_t} + \eta \langle \theta_t, w_t\rangle, 
    \end{align*} 
    where 
    \begin{align*}
       w_t(a) = \begin{cases}
           \hat{v}_t(a)^2,   &\text{if}\ a\in[A], \\
           \theta_t(a)\hat{v}_t(a)^2,  &\text{if}\ a=A+1.
       \end{cases}
    \end{align*}
\end{lemma}

\begin{proof}
    By the optimality of $\theta_{t+1}$, we have 
    \begin{align*}
         \inner{\theta_{t+1}, -\hat{v}_t} + \frac{D_\psi(\theta_{t+1}, \theta_t)}{\eta} \leq \inner{\theta, -\hat{v}_t } + \frac{D_\psi(\theta, \theta_t)}{\eta} - \frac{D_\psi(\theta, \theta_{t+1})}{\eta}. 
    \end{align*}
    Rearranging this gives 
    \begin{align}
         \inner{\theta, \hat{v}_t}
         &\leq \frac{D_\psi(\theta, \theta_t) - D_\psi(\theta, \theta_{t+1})}{\eta} + \inner{\theta_t, \hat{v}_t}  + \inner{\theta_{t} - \theta_{t+1}, - \hat{v}_t } - \frac{D_\psi(\theta_{t+1}, \theta_t)}{\eta}.    \label{eq: omd tmp1}
    \end{align}
    By \pref{lem: hybrid stabiltiy} and the conditions specified in the lemma, we can bound 
    \begin{align*}
       \inner{\theta_{t} - \theta_{t+1}, - \hat{v}_t } - \frac{D_\psi(\theta_{t+1}, \theta_t)}{\eta} \leq \eta \left(\sum_{a=1}^A\theta_t(a)\hat{v}_t(a)^2 + \theta_t(A+1)^2\hat{v}_t(A+1)^2\right), 
    \end{align*}
    which, when combined with \pref{eq: omd tmp1}, finishes the proof.  
\end{proof}

\begin{lemma}\label{lem: lem with correction}
    Let $\Theta\subseteq \Delta_{A+1}$ be convex. Suppose that $\eta|\hat{v}_t(a)|\leq \frac{1}{6}$ for all $a\in[A]$, and $\eta\theta_t(A+1)|\hat{v}_t(A+1)|\leq \frac{1}{6}$. Then the online mirror descent update
    \begin{align*}
        \theta_{t+1} = \argmin_{\theta\in \Theta}\left\{ \inner{\theta, -\hat{v}_t + 2\eta w_t} + \frac{D_\psi(\theta, \theta_t)}{\eta}\right\}
    \end{align*}
    with 
    \begin{align*}
       w_t(a) = \begin{cases}
           \hat{v}_t(a)^2,   &\text{if}\ a\in[A], \\
           \theta_t(a)\hat{v}_t(a)^2,  &\text{if}\ a=A+1,
       \end{cases}
    \end{align*}
    and the hybrid regularizer $\psi$ in \pref{eq: hybrid} ensures for any $\theta\in\Theta$, 
    \begin{align*}
        \inner{\theta, \hat{v}_{t}} \leq \frac{D_\psi(\theta, \theta_t) - D_\psi(\theta, \theta_{t+1})}{\eta}  +  \inner{\theta_{t}, \hat{v}_t} + 2\eta \langle \theta, w_t\rangle. 
    \end{align*}
\end{lemma}
\begin{proof}
    By the optimality of $\theta_{t+1}$, we have 
    \begin{align*}
         \inner{\theta_{t+1}, -\hat{v}_t + 2\eta w_t} + \frac{D_\psi(\theta_{t+1}, \theta_t)}{\eta} \leq \inner{\theta, -\hat{v}_t + 2\eta w_t} + \frac{D_\psi(\theta, \theta_t)}{\eta} - \frac{D_\psi(\theta, \theta_{t+1})}{\eta}. 
    \end{align*}
    Rearranging this gives 
    \begin{align}
         \inner{\theta, \hat{v}_t}
         &\leq \frac{D_\psi(\theta, \theta_t) - D_\psi(\theta, \theta_{t+1})}{\eta} + 2\eta \inner{\theta, w_t}  + \inner{\theta_t, \hat{v}_t} \nonumber \\
         &\qquad + \inner{\theta_{t} - \theta_{t+1}, - \hat{v}_t + \eta w_t} - \frac{D_\psi(\theta_{t+1}, \theta_t)}{\eta} - 2\eta \inner{\theta_t, w_t}.    \label{eq: omd tmp3}
    \end{align}
    By the definition of $w_t(a)$ and the assumptions, it holds that $\forall a\in[A]$, 
    \begin{align*}
        \eta \left| -\hat{v}_{t}(a) + 2\eta w_t(a)\right| 
        &=\eta \left| -\hat{v}_{t}(a) + 2\eta \hat{v}_{t}(a)^2\right| \leq \frac{1}{6} + 2\left(\frac{1}{6}\right)^2 < 1  
    \end{align*}
    and
    \begin{align*}
        &\eta \theta_{t}(A+1)\left| -\hat{v}_{t}(A+1) + 2\eta w_t(A+1)\right| \\
        &=\eta \theta_{t}(A+1)\left| -\hat{v}_{t}(A+1) + 2\eta \theta_{t}(A+1)\hat{v}_{t}(A+1)^2\right| \\
        &\leq \frac{1}{6} + 2\left(\frac{1}{6}\right)^2 < \frac{1}{2}.  
    \end{align*}
    By \pref{lem: hybrid stabiltiy}, we get 
    \begin{align*}
        &\inner{\theta_{t} - \theta_{t+1}, - \hat{v}_t + \eta w_t} - \frac{D_\psi(\theta_{t+1}, \theta_t)}{\eta}  \\
        &\leq \eta \sum_{a=1}^A \theta_{t}(a) (-\hat{v}_{t}(a)+2\eta w_t(a))^2 + \theta_{t}(A+1)^2 (-\hat{v}_{t}(A+1)+2\eta w_t(A+1))^2 \\
        &= \eta \sum_{a=1}^A \theta_{t}(a) (-\hat{v}_{t}(a)+2\eta \hat{v}_t(a)^2)^2 + \theta_{t}(A+1)^2 (-\hat{v}_{t}(A+1)+2\eta \theta_t(A+1)\hat{v}_t(A+1)^2)^2 \\ 
        &\leq \eta \sum_{a=1}^A \theta_{t}(a) \left(\frac{4}{3}\hat{v}_{t}(a)\right)^2  + \eta \theta_{t}(A+1)^2 \left(\frac{4}{3}\hat{v}_{t}(A+1)\right)^2 \tag{by the assumption $\eta|\hat{v}_t(a)|\leq \frac{1}{6}$ for $a\in[A]$ and $\eta \theta_t(A+1)|\hat{v}_t(A+1)|\leq \frac{1}{6}$ }\\
        &\leq 2\eta \left(\sum_{a=1}^A \theta_{t}(a) \hat{v}_{t}(a)^2 + \theta_{t}(A+1)^2 \hat{v}_{t}(A+1)^2\right) \\
        &= 2\eta \langle \theta_t, w_t\rangle. 
    \end{align*}
    Using this in \pref{eq: omd tmp3} finishes the proof. 
\end{proof}

\section{Proof of \pref{thm: warmup thm}} \label{app: proof main thm for warmup}

\begin{proof}
There are two major parts to this proof. In the first part, we show that \pref{eq: surrogate} holds. In the second, we relate the quantity $\sum_{k=1}^{\sqrt{T}}\vhat_{k}(a)$ to the static and dynamic regret for $a\in \crl{1,2}$.

\paragraph{Applying \pref{lem: log-bar omd smooth} to obtain \pref{eq: surrogate}.}
To invoke \pref{lem: log-bar omd smooth}, we verify that $\eta\theta_{k}(a)|\vhat_{k}(a)| \leq 1/6$ for $k\in [\sqrt{T}]$ and $a\in \crl{1,2}$. Recall the definition of $\vhat_k$ and $\lhat_k$:
\begin{align*}
\vhat_k = p_{k,\xx}\hat{\ell}_{k,\xx} + p_{k,\oo}\hat{\ell}_{k,\oo}  \quad\mbox{and}\quad 
\begin{cases}
    \hat{\ell}_{k,\xx} = \frac{\ind\{i_k=\oo\}}{p_{k,\oo}}\begin{bmatrix}
        0 \\
        \frac{1}{2}T^{1/4}\hat{L}_k
   \end{bmatrix}, \\[14pt]
   \hat{\ell}_{k,\oo} = \frac{\ind\{i_k=\oo\}}{p_{k,\oo}}\begin{bmatrix}
        T^{3/4} - \frac{1}{2}u\sqrt{T}\hat{L}_k \\
        \sqrt{T}
   \end{bmatrix}.
\end{cases}
\end{align*}
Thus we have for $a=1$ with the choice of $\eta \leq 1/(2T)$,
\begin{align*}
    \eta\theta_{k}(1) |\vhat_{k}(1)| &= \eta \theta_{k}(1) |p_{k,\xx}\lhat_{k,\xx}(1) + p_{k,\oo}\lhat_{k,\oo}(1)| \\
    &= \eta \theta_{k}(1) \abs*{  p_{k,\oo} \frac{\ind\{i_k=\oo\}}{p_{k,\oo}} \prn*{ T^{3/4} - \frac{1}{2}u\sqrt{T}\hat{L}_k}}\\
    &\leq \eta\theta_{k}(1) \abs*{  T^{3/4} - \frac{1}{2}u\sqrt{T}\hat{L}_k}\\
    &\leq \frac{1}{6},
\end{align*}
where the last inequality uses the fact that $\Lhat_k$ is an underestimator of $L_k$, which implies $\Lhat_k\leq L_k\leq \sqrt{T}$.
Meanwhile, for $a=2$, we have
\begin{align*}
    \eta\theta_{k}(2) |\vhat_{k}(2)| &= \eta \theta_{k}(2) |p_{k,\xx}\lhat_{k,\xx}(2) + p_{k,\oo}\lhat_{k,\oo}(2)| \\
    &= \eta \theta_{k}(2) \abs*{ p_{k,\xx}   \frac{\ind\{i_k=\oo\}}{p_{k,\oo}} \frac{1}{2}T^{1/4}\hat{L}_k +  p_{k,\oo}  \frac{\ind\{i_k=\oo\}}{p_{k,\oo}} \sqrt{T}  }\\
    &\leq  \eta \theta_{k}(2) \abs*{ \frac{p_{k,\xx}}{2p_{k,\oo}} T^{1/4}\hat{L}_k}   + 1/12.
\end{align*}
Recall the definition of $p_{k,\oo}$, $p_{k,\xx}$ from \pref{eq: p choice warm},
\begin{align*}
    (p_{k,\xx}, p_{k,\oo}) =  \frac{\left(u\theta_{k}(1)\sqrt{T}, \ \ \theta_{k}(2)T^{1/4}\right)}{\theta_{k}(2)T^{1/4} + u\theta_{k}(1)\sqrt{T}}.
\end{align*}
Then by plugging definitions into the previous inequality, we further have
\begin{align*}
    \eta\theta_{k}(2) |\vhat_{k}(2)| &\leq \eta \theta_{k}(2) \abs*{ \frac{p_{k,\xx}}{2p_{k,\oo}} T^{1/4}\hat{L}_k}  + 1/12\\
     &= \eta \theta_{k}(2) \frac{u\theta_{k}(1)\sqrt{T}}{2\theta_{k}(2)T^{1/4}}  T^{1/4}\hat{L}_k   + 1/12\\
     &=u \eta \theta_{k}(1) \sqrt{T} \Lhat_k/2   + 1/12\\
     &\leq 1/6,
\end{align*}
where the last inequality is by the choice of $\eta\leq 1/(2T)$. Now we can invoke \pref{lem: log-bar omd smooth} with $k$ as the index with a total number of $\sqrt{T}$ steps. Thus we have
\begin{align}
    \sum_{k=1}^K \hat{v}_{k}(a) \leq 
     \frac{2}{\eta} \ln \frac{1}{\alpha}   + 2\eta \sum_{k=1}^K \theta_{k}(a) \hat{v}_{k}(a)^2 + \sum_{k=1}^K \inner{\theta_{k}, \hat{v}_k} +  \alpha\sum_{k=1}^K \left\langle e_a - \frac{\mathbf{1}}{A}, \hat{v}_k - 2\eta\theta_k \hat{v}_k^2 \right\rangle, \label{ineq:smooth}
\end{align}
for all $a\in \crl{1,2}$ and $\alpha\in (0,1]$. Note the following bounds:
\begin{align*}
\sum\limits_{k=1}^{K}\En[\abs*{\vhat_k(1)}] \leq  \sum\limits_{k=1}^{K}\abs*{  T^{3/4} - \frac{1}{2}u\sqrt{T}\hat{L}_k} \leq T^{3/2},
\end{align*}
\begin{align*}
    \sum\limits_{k=1}^{K}\En[\abs*{\vhat_k(2)}] &= \sum\limits_{k=1}^{K}  \En\brk*{\abs*{ p_{k,\xx}   \frac{\ind\{i_k=\oo\}}{p_{k,\oo}} \frac{1}{2}T^{1/4}\hat{L}_k +  p_{k,\oo}  \frac{\ind\{i_k=\oo\}}{p_{k,\oo}} \sqrt{T}  }}  \\
    &\leq  \sum\limits_{k=1}^{K} 2T^{3/4} = 2T^{5/4},
\end{align*}
and 
\begin{align*}
\sum\limits_{k=1}^{K} \En[\abs*{\eta\theta_k(a) \hat{v}_k^2(a)}] \leq \sum\limits_{k=1}^{K} \En[\abs*{ \hat{v}_k(a)}/6] \leq T^{3/2},
\end{align*}
for $a=1$ or $2$.
Thus using the bounds above for the last term in \pref{ineq:smooth} together with the choice of $\eta=c\cdot 1/T$ with $c$ small enough and $\alpha = T^{-3/2}/8$, we have for $a\in \crl{1,2}$
\begin{align*}
    \sum_{k=1}^{K}\hat{v}_{k}(a) &\leq O(T\log T) +    2\eta\sum_{k=1}^{K} \theta_{k}(a) \hat{v}_{k}(a)^2 + \sum_{k=1}^{K} \theta_k^\top \hat{v}_k\\
    &\leq O(T\log T) +    \sum_{k=1}^{K} |2\eta\theta_{k}(a) \hat{v}_{k}(a)||\hat{v}_{k}(a)| + \sum_{k=1}^{K} \theta_k^\top \hat{v}_k \\
    &\leq  O(T\log T) +    \frac{1}{2}\sum_{k=1}^{K} |\hat{v}_{k}(a)| + \sum_{k=1}^{K} \theta_k^\top \hat{v}_k,
\end{align*}
where the last inequality follows again by $\eta\theta_{k}(a)|\vhat_{k}(a)| \leq 1/6$ for $k\in [\sqrt{T}]$ and $a\in \crl{1,2}$. 

\paragraph{Upper bounding the static and dynamic regret through $\sum_{k=1}^{\sqrt{T}}\vhat_{k}(a)$} 
Recall \pref{eq: include u}, for $a=1$, we have
\begin{align*}
    \frac{1}{2}\ind\{i_k=\oo\}\left(T^{3/4} - \frac{3}{2}u\sqrt{T}\hat{L}_k\right) \leq \hat{v}_{k}(1) - \frac{1}{2}|\hat{v}_{k}(1)|.
\end{align*}
Then combine the above with \pref{eq: Sreg bound} and \pref{eq: Dreg bound}, we have the static regret is upper bounded by
\begin{align*}
    \E\left[\sum_{k=1}^{K} \left(p_{k,\xx}\SReg_{k,\xx} + p_{k,\oo}\SReg_{k,\oo}\right)\right] &\leq \E\left[\sum_{k=1}^{K} p_{k,\oo}\left(T^{1/4} - \frac{1}{2}L_k\right)\right]\\
    &\leq  \frac{1}{\sqrt{T}}\cdot\sum_{k=1}^{K} \E\brk*{ \ind\{i_k=\oo\}\left(T^{3/4} - \frac{3}{2}u\sqrt{T}\hat{L}_k\right)}\\
    &\leq \frac{2}{\sqrt{T}} \cdot\sum_{k=1}^{K}  \E\brk*{ \hat{v}_{k}(1) - \frac{1}{2}|\hat{v}_{k}(1)|  }.
\end{align*}
Combine further with \pref{eq: surrogate} shown in the first part of this proof, we have
\begin{align*}
     \E\left[\sum_{k=1}^{K} \left(p_{k,\xx}\SReg_{k,\xx} + p_{k,\oo}\SReg_{k,\oo}\right)\right] &\leq \frac{1}{\sqrt{T}} \prn*{O(T) +  \En\brk*{\sum_{k=1}^{K} \theta_k^\top \hat{v}_k }}\\
     &\leq O(\sqrt{T}),
\end{align*}
where the last inequality is obtained from \pref{eq: residual}.
Similarly, for the dynamic regret, we have by \pref{eq: Sreg bound} and \pref{eq: Dreg bound},
\begin{align*}
    \E\left[p_{k,\xx}\DReg_{k,\xx} + p_{k,\oo}\DReg_{k,\oo}\right] &\leq \E\left[p_{k,\xx} L_k/2  + p_{k,\oo} T^{1/4}\right].
\end{align*}
We further note that  $ L_k \leq 2 T^{1/4} + \En\brk*{\Lhat_k}$. Recall the definition of $\lhat_{k,\xx}, \lhat_{k,\oo}$ from \pref{eq: warmup estimator}, consequently, we have 
\begin{align*}
     \E\left[p_{k,\xx} L_k/2  + p_{k,\oo} T^{1/4}\right] &\leq \frac{1}{T^{1/4}} \prn*{ \sqrt{T} + \En\brk*{    p_{k,\xx}\lhat_{k,\xx}(2) + p_{k,\oo}\lhat_{k,\oo}(2)       }  }\\
     &= \frac{1}{T^{1/4}} \prn*{ \sqrt{T} + \En\brk*{ \vhat_{k}(2)  }  }.
\end{align*}
We also have from \pref{eq: surrogate} and $\vhat_{k}(2)\geq 0$ for all $k\geq 0$ that 
\begin{align*}
    \En\brk*{\sum_{k=1}^K\vhat_{k}(2)} \leq  O(T) + 2 \E\left[\sum_{k=1}^{K} \theta_k^\top \hat{v}_k \right] \leq O(T).
\end{align*}
Combine the above three displayed inequalities, we have
\begin{align*}
      \E\left[\sum_{k=1}^K \prn*{ p_{k,\xx}\DReg_{k,\xx} + p_{k,\oo}\DReg_{k,\oo}} \right]\leq O(T^{3/4}).
\end{align*}
This concludes our proof.


\end{proof}

\section{Guarantees of \expalg in $\oo$ Epochs (\pref{lemma: sreg under ooo}--\pref{lem: dreg in o})}\label{app: expalg}

In this section, we introduce \expalg (\pref{alg: auer et l - det}) that is executed in an $\oo$ epoch. It is detection-based dynamic regret minimization algorithm similar to \cite{karnin2016multi} and \cite{auer2019adaptively}. Our algorithm is simpler than theirs as we assume deterministic losses and thus changes can be detected with a single sample.

The algorithm proceeds in \emph{blocks}. Each block $m$ begins with Phase~1, during which each arm is sampled exactly once. At the end of Phase~1, the algorithm computes the sub-optimality gap $\Delta^{m}(a)$ (\pref{line: calculate gap}) and the corresponding sampling probability $\prob\super{m}(a)$ (\pref{line: gap dep exp}). In Phase~2, arms are sampled according to the distribution $\prob\super{m}$, and the block terminates as soon as a change in the loss of any arm is detected (\pref{line: change detected}).

The sampling probabilities are chosen to satisfy $\prob\super{m}(a)\leq \frac{\rho}{A \Delta\super{m}(a)}$ which ensures that the per-round exploration cost is bounded by $\sum_a \prob\super{m}(a)\Delta\super{m}(a)\leq \sum_a \frac{\rho}{A\Delta\super{m}(a)} \Delta\super{m}(a)\leq \rho$. While the same exploration cost could also be achieved by sampling all arms uniformly with probability~$\frac{\rho}{A}$, the inverse-gap sampling probability is crucial: For arms with smaller gaps, as the learner accumulates less negative regret against them, their changes must be more quickly detected to prevent regret from detection delay. This is crucial in establishing \pref{eq: o bound 1} in \pref{lemma: sreg under ooo}.   

Below, we bound the static regret incurred by \expalg in \pref{lemma: sreg under ooo}, the estimation error $|L_k(a) - \Lhat_k(a)|$ in \pref{lem: L error}, and the dynamic regret in \pref{lem: dreg in o}. 

\begin{lemma}\label{lemma: sreg under ooo}
    With probability at least $1-2\delta$, for all $k\in[K]$, choosing $\oo$ in epoch $k$ ensures for all~$a\in[A]$: 
    \begin{align}
        \SReg_{k,\oo}(a)  &\leq  \min\left\{5\rho\sqrt{AT} +\frac{7A\hat{s}_k\log(AT/\delta) }{\rho}, \sqrt{AT}\right\} - \frac{3}{4}\Lhat_k(a),   \label{eq: o bound 1} \\ 
        \SReg_{k,\oo}(a)  &\leq \min\left\{4\rho\sqrt{AT} + \frac{7A(\hat{s}_k + \ind\{s_k>0\})\log(AT/\delta)}{\rho}, \sqrt{AT}\right\} - \Lhat_k(a). \label{eq: o bound 2}
    \end{align}
    where $\Lhat_k$ and $\hat{s}_k$ are outputs of \expalg. 
\end{lemma}
\begin{proof}
    Recall that $\tau\super{m}$ is the time index at the beginning of block $m$ of \expalg (defined in \pref{line: starting index}). For notational simplicity, define $\tau\super{\hat{s}_k+1} = t_k+1$. Then we have 
\allowdisplaybreaks
    \begin{align}
        \SReg_{k,\oo}(a) 
        &= \sum_{t=e_k}^{t_k} (c_t(a_t) - c_t(a))  \notag \\ 
        &\leq (\hat{s}_k+1)A + \sum_{m=0}^{\hat{s}_k}\sum_{t=\tau\super{m}+A}^{\tau\super{m+1}-1} \left(c_t(a_t) - c_t(a)\right) \notag \\
        &= (\hat{s}_k+1)A + \underbrace{\sum_{m=0}^{\hat{s}_k}\sum_{t=\tau\super{m}+A}^{\tau\super{m+1}-1} \left(\hat{c}_t(a_t) - \min_{a'}\hat{c}_t(a')\right)}_{\term_1} + \underbrace{\sum_{m=0}^{\hat{s}_k}\sum_{t=\tau\super{m}+A}^{\tau\super{m+1}-1} \left(\min_{a'}\hat{c}_t(a') - \hat{c}_t(a)\right)}_{\term_2} \notag\\
        &\qquad \quad + \underbrace{\sum_{m=0}^{\hat{s}_k}\sum_{t=\tau\super{m}+A}^{\tau\super{m+1}-1} (c_t(a_t) - \hat{c}_t(a_t))}_{\term_3} + \underbrace{\sum_{m=0}^{\hat{s}_k}\sum_{t=\tau\super{m}+A}^{\tau\super{m+1}-1} (\hat{c}_t(a) - c_t(a))}_{\term_4}.
        \label{eq: sreg o decompose}
    \end{align}
    Now we bound the four terms above. Below, denote $\tdelta = \frac{\delta}{AT}$. 
    
    \paragraph{Bounding $\term_1$. } 
    With probability at least $1-\tdelta$, 
\begin{align}
    \term_1&=\sum_{m=0}^{\hat{s}_k}\sum_{t=\tau\super{m}+A}^{\tau\super{m+1}-1} \left(\hat{c}_t(a_t) - \min_{a'} \hat{c}_t(a')\right) \notag\\
    &=\sum_{m=0}^{\hat{s}_k} \sum_{t=\tau\super{m}+A}^{\tau\super{m+1}-1} \sum_{a\in[A]} \ind\{ a_t=a \}\Delta\super{m}(a)  \tag{in the Phase~2 of block $m$, $\hat{c}_t(a) - \min_{a'}\hat{c}_t(a')=\Delta\super{m}(a)$} \\
    &\leq \frac{3}{2}\sum_{m=0}^{\hat{s}_k} \sum_{t=\tau\super{m}+A}^{\tau\super{m+1}-1} \sum_{a\in[A]} \prob\super{m}(a)\Delta\super{m}(a) + 4\log(1/\tdelta) \tag{by \pref{lem: freedman+}} \\
    &\leq \frac{3}{2}\sum_{m=0}^{\hat{s}_k} \sum_{t=\tau\super{m}+A}^{\tau\super{m+1}-1} \sum_{a\in[A]} \frac{\exprate}{A} + 4\log(1/\tdelta) \tag{$\prob\super{m}(a)\leq \frac{\exprate}{A\Delta\super{m}(a)}$ by definition} \\
    &\leq \frac{3}{2}\rho\sqrt{AT} + 4\log(1/\tdelta). \label{eq: bounding exploration term}
\end{align}
    
    \paragraph{Bounding $\term_2$. }
    By the definition of $\Lhat_k(a)$ in \pref{line: hatL definition} of \pref{alg: auer et l - det}: 
    \begin{align*}
        \term_2 &= \sum_{m=0}^{\hat{s}_k}\sum_{t=\tau\super{m}+A}^{\tau\super{m+1}-1} \left(\min_{a'}\hat{c}_t(a') - \hat{c}_t(a)\right) \\
        &\leq \sum_{m=0}^{\hat{s}_k}\sum_{t=\tau\super{m}}^{\tau\super{m+1}-1} \left(\min_{a'}\hat{c}_t(a') - \hat{c}_t(a)\right) + (\hat{s}_k+1)A\\ 
        &= -\Lhat_k(a) + (\hat{s}_k+1)A. 
    \end{align*}
    \paragraph{Bounding $\term_3$. } 
    $\term_3$ is at most $\hat{s}_k + 1$ because $c_t(a_t) \neq \hat{c}_t(a_t)$ can only happen once in the Phase~2 of every block $m$. 
    
    \paragraph{Bounding $\term_4$. } 
    To bound $\term_4$, note that whenever $c_t(a) \neq \hat{c}_t(a)$, with probability $\prob\super{m}(a)$, the change will be detected and block $m$ will terminate. Therefore, for any $m$, any $a$ and any $N$, 
    \begin{align*}
       \Pr\left\{\sum_{t=\tau\super{m}+A}^{\tau\super{m+1}-1} \ind\{\hat{c}_t(a)\neq c_t(a)\} \geq N \right\} \leq \left(1-\prob\super{m}(a)\right)^{N-1}. \tag{the probability of failing to detect the change consecutively for $N-1$ times}
    \end{align*}
    Setting the right-hand side above as $\tdelta$, we get: with probability at least $1-\tdelta$, 
    \begin{align*}
        \sum_{t=\tau\super{m}+A}^{\tau\super{m+1}-1} \ind\{\hat{c}_t(a)\neq c_t(a)\} \leq 1 - \frac{\log(1/\tdelta)}{\log(1-\prob\super{m}(a))} \leq 1+ \frac{\log(1/\tdelta)}{\prob\super{m}(a)} \leq 1 + \left(A + \frac{A}{\exprate}\Delta\super{m}(a)\right)\log(1/\tdelta), 
    \end{align*}
    where the last inequality is by the definition of $\prob\super{m}(a)$. Thus, with probability at least $1-\sqrt{AT}\tdelta$, 
    \begin{align*}
         \term_4
         &\leq \sum_{m=0}^{\hat{s}_k}\sum_{t=\tau\super{m}+A}^{\tau\super{m+1}-1} \ind\{\hat{c}_t(a)\neq c_t(a)\} \\ 
         &\leq \sum_{m=0}^{\hat{s}_k} \left(1+A+\frac{A}{\exprate}\Delta\super{m}(a) \right)\log(1/\tdelta)\\
         &\leq (\hat{s}_k+1)(A+1)\log(1/\tdelta) + \frac{A\log(1/\tdelta)}{\exprate}\sum_{m=0}^{\hat{s}_k}\Delta\super{m}(a). 
    \end{align*}
    Combining the four terms above with \pref{eq: sreg o decompose}, we get: with probability at least $1-2\sqrt{AT}\tdelta$, 
    \begin{align}
        &\SReg_{k,\oo}(a) \notag\\
        &\leq \frac{3}{2}\rho\sqrt{AT} + 4\log(1/\tdelta) + 3(\hat{s}_k+1)(A+1)\log(AT/\delta) - \Lhat_k(a) + \frac{A\log(1/\tdelta)}{\exprate} \sum_{m=0}^{\hat{s}_k} \Delta\super{m}(a) \notag\\
        &\leq 4\rho\sqrt{AT} + 5\hat{s}_k A\log(1/\tdelta) - \Lhat_k(a) + \frac{A\log(1/\tdelta)}{\exprate} \sum_{m=0}^{\hat{s}_k} \Delta\super{m}(a)
        \label{eq: general Sreg bound o} 
    \end{align}
    where in the last inequality we simplify the bound using $\exprate= 4\log(1/\tdelta)\sqrt{SA/T}$ and $A\geq 2$.   
    Below, we provide refined bounds when $\hat{s}_k$ is zero and when $s_k$ is zero. 

\paragraph{Refinement when $\hat{s}_k=0$. } When $\hat{s}_k=0$, \pref{eq: general Sreg bound o}  becomes 
    \begin{align}
        \SReg_{k,\oo}(a)\leq 4\rho\sqrt{AT}  - \Lhat_k(a) + \frac{A\log(1/\tdelta)}{\exprate} \Delta\super{0}(a).   \label{eq: hats = 0refine}
    \end{align}
    Also, when $\hat{s}_k=0$, 
    \begin{align}
        \Lhat_k(a) 
        =  \sum_{t=e_k}^{t_k} \left(\hat{c}_{t}(a) - \min_{a'\in [A]} \hat{c}_{t}(a')\right) 
        &\geq \sum_{t=e_k +A}^{t_k-1}  \Delta\super{0}(a) \geq \sqrt{AT}\Delta\super{0}(a) - (A+1). \label{eq: hatL lower bound}  
    \end{align}
    By our choice of $\exprate= 4\log(1/\tdelta)\sqrt{SA/T}$, the last term in \pref{eq: hats = 0refine} can be upper bounded as 
    \begin{align*}
       \frac{A\log(1/\tdelta)}{\exprate} \Delta\super{0}(a)
       &\leq \frac{1}{4}\sqrt{\frac{AT}{S}}\Delta\super{0}(a) \\
       &\leq \frac{1}{4}\left(\Lhat_k(a) + (A+1)\right).    \tag{using \pref{eq: hatL lower bound}}
    \end{align*}
    Plugging this in \pref{eq: hats = 0refine}, we thus refine the bound when $\hat{s}_k=0$ as 
    \begin{align}
        \SReg_{k,\oo}(a) \leq 5\exprate\sqrt{AT} - \frac{3}{4}\Lhat_k(a).   \label{eq: refinement in hats = 0}
    \end{align}
    We obtain \pref{eq: o bound 1} by combining the general bound \pref{eq: general Sreg bound o} and the special bound \pref{eq: refinement in hats = 0} for $\hat{s}_k=0$, using $A\geq 2$ and $\Delta\super{m}(a)\leq 1$ and that all positive terms are bounded by $|\calI_k|=\sqrt{AT}$ trivially, and furthermore using an union bound over all epochs and all actions. 
    
    \paragraph{Refinement when $s_k=0$. }  When $s_k=0$ (which implies $\hat{s}_k=0$), we have $\hat{c}_t(a) = c_t(a)$ for all $a\in[A]$ and for all $t$ in Phase~2. That means $\term_3=\term_4=0$ in \pref{eq: sreg o decompose}. Hence, in this case, with probability at least $1-\tdelta$, 
    \begin{align}
        \SReg_{k,\oo}(a)\leq \frac{3}{2}\exprate\sqrt{AT} + 4\log(1/\tdelta) + A - \Lhat_k(a). \label{eq: s_k = 0 special bound}
    \end{align}
    We obtain \pref{eq: o bound 2} by combining the general bound \pref{eq: general Sreg bound o} and the special bound \pref{eq: s_k = 0 special bound} for $s_k=0$. 
\end{proof}

\begin{lemma}\label{lem: L error}
    With probability at least $1-\delta$, for all $k\in[K]$, choosing $\oo$ in epoch $k$ ensures 
    \begin{align*}
        \forall a\in[A], \qquad \big|L_k(a) - \Lhat_k(a)\big|\leq  A + \frac{3A(\hat{s}_k + \ind\{s_k>0\})\log(AT/\delta)}{\rho}. 
    \end{align*}
\end{lemma}

\begin{proof}
    For notational simplicity, define $\tau\super{\hat{s}_k+1} = t_k+1$. Then 
    \begin{align}
        \big|L_k(a) - \Lhat_k(a)\big| 
       &= \left| \sum_{m = 0}^{\hat{s}_k}  \sum_{t=\tau\super{m}}^{\tau\super{m+1}-1} \left(\left(c_t(a) - \min_{a'}c_t(a')\right) - \left(\hat{c}_t(a) - \min_{a'}\hat{c}_t(a')\right)\right)\right| \tag{by the definitions of $L_k(a)$ and $\Lhat_k(a)$} 
        \\& \leq ({\hat{s}_k+1})A +  \sum_{m = 0}^{\hat{s}_k}  \sum_{t=\tau\super{m}+A}^{\tau\super{m+1}-1} \ind\{\exists a', c_t(a')\neq \hat{c}_t(a')\}. \label{eq: imply upper}
    \end{align}
    In the Phase~2 of every block $m$, whenever there exists $a'$ such that $c_t(a') \neq \hat{c}_t(a')$, with probability $\prob\super{m}(a')$, the change will be detected and the block will terminate. Therefore, 
    \begin{align*}
        \Pr\left\{ \sum_{t=\tau\super{m}+A}^{\tau\super{m+1}-1}\ind\{ \exists a', c_t(a')\neq \hat{c}_t(a') \} 
        \geq N \right\} &\leq \left(1- \min_{a'} \prob\super{m}(a') \right)^{N-1} \leq \left(1-\frac{\rho}{A}\right)^{N-1}. \tag{the probability of failing to detect the change for $N-1$ times}
    \end{align*}
    Setting the right-hand side above as $\tdelta$, we get: with probability at least $1-\tdelta$, 
    \begin{align}
       \sum_{t=\tau\super{m}+A}^{\tau\super{m+1}-1}\ind\{ \exists a', c_t(a')\neq \hat{c}_t(a') \} \leq 1 - \frac{\log (1/\tdelta)}{\log(1-\frac{\rho}{A})}\leq 1 + \frac{A\log(1/\tdelta)}{\rho}. \label{eq: upper 2}
    \end{align}
    Therefore, with probability at least $1-\sqrt{AT}\tdelta$, \pref{eq: upper 2} holds for all $m$ and implies   
    \begin{align*}
        \forall a\in[A], \quad \big|L_k(a) - \Lhat_k(a)\big|\leq (\hat{s}_k+1)\left(A+1 + \frac{A\log(1/\tdelta)}{\rho}\right) \leq \frac{3A(\hat{s}_k+1)\log(1/\tdelta)}{\rho} 
    \end{align*}
    by \pref{eq: imply upper}. 
    In the special case where $s_k=0$ (no switch in epoch $k$), we have $\hat{s}_k=0$ and $c_t(a') = \hat{c}_t(a')$ for all $a'\in[A]$ in Phase~2. By \pref{eq: imply upper}, it holds that $|L_k(a) - \Lhat_k(a)|\leq A$. Combining the cases for $s_k=0$ and $s_k>0$, choosing $\tdelta=\frac{\delta}{AT}$, and using a union bound gives the desired bound. 
\end{proof}

\begin{lemma}\label{lem: dreg in o}
With probability at least $1-3\delta$, for all $k\in[K]$, choosing $\oo$ in epoch $k$ ensures 
    \begin{align*}
        \DReg_{k,\oo} \leq 5\rho\sqrt{AT} + \frac{10A(\hat{s}_k + \ind\{s_k>0\})\log(AT/\delta)}{\rho}. 
    \end{align*}
\end{lemma}
\begin{proof}
    By the definition of $L_k(a)$, it holds that for any $a\in[A]$,  
    \begin{align*}
        \DReg_{k,\oo} = \SReg_{k,\oo}(a) + L_k(a).    
    \end{align*}
   Combining \pref{lemma: sreg under ooo} and \pref{lem: L error}, we can bound it as  
   \begin{align*}
      \DReg_{k,\oo} 
      &\leq 4\rho\sqrt{AT} + \frac{7A(\hat{s}_k + \ind\{s_k>0\})\log(AT/\delta)}{\rho} - \Lhat_k(a) + L_k(a) \\
      &\leq 5\rho\sqrt{AT} + \frac{10A(\hat{s}_k + \ind\{s_k>0\})\log(AT/\delta)}{\rho} 
   \end{align*}
   with probability at least $1-3\delta$. 
\end{proof}

\section{Mirror Descent Analysis} \label{app: dynamic regret in x}


\subsection{In $i_k=\xx$ Epochs (\pref{lem: mirror descent for x}, \pref{lem:one-interval-omd})}
\begin{lemma}\label{lem: mirror descent for x}
    With probability at least $1-\delta$, for all epochs $\calI_k=[e_k, t_k]$ where $i_k=\xx$, 
       \begin{align*}
          &\sum_{t\in\calI_k} \inner{\theta,\vhat_t} \leq \frac{D_\psi(\theta, \theta_{e_k}) - D_\psi(\theta, \theta_{t_k+1})}{\eta} + \frac{5}{2}\sqrt{AT}\log(AT/\delta). 
       \end{align*}
       for any $\theta\in\Theta$. 
\end{lemma}
\begin{proof}
   We will use \pref{lem: lem without correction} for rounds in epoch $k$. 
    We first verify the condition in \pref{lem: lem without correction}. We have
    \begin{align*}
        \eta \vhat_{t}(a) &= \eta \sqrt{\frac{T}{A}} \cdot \prn*{ c_{t}(a_t) - \frac{\ind(a_t=a)}{q_{t}(a)+\gamma} c_{t}(a)} \leq \eta \sqrt{ \frac{T}{A} }
        \leq 1.
    \end{align*}    
    Meanwhile $\eta \theta_t(A+1)\vhat_t(A+1) =0 \leq 1/2$. Thus, we apply \pref{lem: lem without correction} and sum the bound over $t\in [e_k,t_k]$ to obtain for any $\theta\in\Theta$,
    \begin{align}
        \sum_{t\in \calI_k} \inner{\theta, \hat{v}_{t}} \leq \frac{D_\psi(\theta, \theta_{e_k}) - D_\psi(\theta, \theta_{t_k+1})}{\eta}  + \sum_{t\in \calI_k} \inner{\theta_t, \hat{v}_t} + 2\eta \sum_{t\in \calI_k} \inner{\theta_t, w_t}, 
        \label{eq: apply lemma 6}
    \end{align}
    where $ w_t(a) = 
        \begin{cases}
            \vhat_t(a)^2\ind(a\leq A), \\
            \theta_t(a)\vhat_t(a)^2\ind(a=A+1).
        \end{cases}$ Since $\vhat(A+1) = 0$, we can further simply as $w_t(a) = \vhat_t^2(a)$.
    Now we bound the terms on the right-hand side of \pref{eq: apply lemma 6}. Below, we denote $\tdelta=\frac{\delta}{AT}$. 
    
    \paragraph{Bounding $\sum_{t\in\calI_k} \inner{\theta_t, \vhat_t}$. }
    By the definition of $\vhat_t$ and that $\frac{\theta_t(a)}{q_t(a)}=\sum_{a'\in[A]}\theta_t(a')$ for all $a\in[A]$, we have
    \begin{align}
        &\inner{\theta_t,\vhat_{t}} \\
        &= \sqrt{\frac{T}{A}} \cdot   \sum_{a\in[A]} 
\theta_t(a) \prn*{ c_{t}(a_t) - \frac{\ind(a_t=a)c_{t}(a)}{q_{t}(a) + \gamma} }   \nonumber \\
&= \sqrt{\frac{T}{A}} \left[\sum_{a\in[A]} \theta_t(a)  c_t(a_t) -  \sum_{a\in[A]} \left(\theta_t(a) + \gamma \frac{\theta_t(a)}{q_t(a)}\right)  \frac{\ind(a_t=a)c_t(a)}{q_t(a) + \gamma} + \sum_{a\in[A]} \gamma \frac{\theta_t(a)}{q_t(a)}  \frac{\ind(a_t=a)c_t(a)}{q_t(a) + \gamma}\right] \notag\\
&= \sqrt{\frac{T}{A}}\left( \sum_{a\in[A]} \theta_t(a) \right) \left[ c_t(a_t) -  \sum_{a\in[A]} \ind\{a_t=a\}c_t(a) + \sum_{a\in[A]}\frac{\gamma \ind\{a_t=a\}c_t(a)}{q_t(a)+\gamma}\right] \tag{$\frac{\theta_t(a)}{q_t(a)}=\sum_{a'\in[A]}\theta_t(a')$}\\
&\leq \gamma\sqrt{\frac{T}{A}} \sum_{a\in[A]} \frac{\ind\{a_t=a\}c_t(a)}{q_t(a)+\gamma}. \label{eq: resterm}
    \end{align}
    Thus, with probability at least $1-\tdelta$, 
\begin{align}
    \sum_{t\in \calI_k} \inner{\theta_t, \vhat_t} &\leq \gamma \sqrt{\frac{T}{A}}\sum_{a\in[A]} \sum_{t\in \calI_k}  \frac{\ind\{a_t=a\}c_t(a)}{q_t(a)+\gamma} \notag\\
    &=  \gamma \sqrt{\frac{T}{A}}\sum_{a\in[A]} \sum_{t\in \calI_k} \left( \frac{\ind\{a_t=a\}c_t(a)}{q_t(a)+\gamma} - c_t(a)\right) + \gamma \sqrt{\frac{T}{A}}\sum_{a\in[A]} \sum_{t\in \calI_k} c_t(a) \notag \\
    &\leq \gamma \sqrt{\frac{T}{A}} \frac{A\log(1/\tdelta)}{2\gamma} + \gamma \sqrt{\frac{T}{A}} A |\calI_k|   \tag{\pref{lem: IX lemma}} \\
    &\leq \frac{1}{2}\sqrt{AT}\log(1/\tdelta) + \sqrt{AT}\log(1/\tdelta) \tag{$\gamma=\frac{\log(1/\tdelta)}{\sqrt{AT}}$} \\
    &\leq \frac{3}{2}\sqrt{AT}\log(1/\tdelta). 
    \label{eq: theta v in x}
\end{align}
    \paragraph{Bounding $\sum_{t\in\calI_k} \inner{\theta_t, w_t}$. }
    For any $t\in [T]$, we denote the expectation $\En_t[\cdot] = \En[\cdot \mid{} \cH_t]$ with $\cH_t$ being the history of the interactions before time step $t$.
For any $t\in\calI_k$, 
    \begin{align}
        \En_t \brk*{\sum_{a\in[A]}\theta_{t}(a)\vhat_{t}(a)^2} &= \sum_{a\in[A]}\En_t\brk*{\theta_{t}(a) \frac{T}{A} \prn*{c_{t}(a_t) -  \frac{\ind(a_t=a)}{q_{t}(a) + \gamma} c_{t}(a)  }^2     }   \nonumber \\
        &\leq \sum_{a\in[A]} \frac{2T}{A}\En_t \brk*{ \theta_{t}(a)\prn*{c_t(a_t)^2 + \frac{\ind(a_t=a)}{(q_{t}(a)+\gamma)^2} c_{t}(a)^2   }}    \tag{$(x+y)^2\leq 2(x^2+y^2)$}\\
        &\leq \sum_{a\in[A]} \frac{2T}{A}\En_t \brk*{ \theta_{t}(a)\prn*{1 + \frac{1}{q_{t}(a)}   }}    \nonumber \\
        &= \sum_{a\in[A]} \frac{2T}{A} \En_t\brk*{ \theta_{t}(a)\left( 1+  \frac{\sum_{a'\in [A]}\theta_{t}(a')}{\theta_{t}(a)} \right)}    \tag{by the definition of $q_t(a)$}\\
        &\leq \sum_{a\in[A]} \frac{4T}{A}\\
        &= 4T. \label{eq: square term} 
    \end{align}
Furthermore, with probability $1$, 
\begin{align*}
   \sum_{a\in[A]}\theta_t(a)\vhat_t(a)^2 
   &\leq \sum_{a\in[A]} \theta_t(a)\times \frac{T}{A} \times \frac{1}{(q_t(a)+\gamma)^2}\leq \sum_{a\in[A]} \frac{T}{A\gamma} = \frac{T}{\gamma}.   \tag{$\theta_t(a)\leq q_t(a)$}
\end{align*}
By \pref{lem: freedman+}, with probability at least $1-\tdelta$, 
\begin{align}
   \sum_{t\in\calI_k} \inner{\theta_t, w_t} &=\sum_{t\in \calI_k} \sum_{a\in[A]}\theta_t(a) \vhat_t(a)^2 \notag\\ 
   &\leq \frac{3}{2} \sum_{t\in \calI_k} \sum_{a\in[A]} \E_t\left[\theta_t(a) \vhat_t(a)^2\right] + \frac{4T\log(1/\tdelta)}{\gamma} \notag\\
   &\leq \frac{3}{2}\sqrt{AT}\times 4T + \frac{4T\log(1/\tdelta)}{\frac{4\log(1/\tdelta)}{\sqrt{AT}}} = 7T\sqrt{AT}. \label{eq: sta in x A+1} 
\end{align}
Using \pref{eq: theta v in x} and \pref{eq: sta in x A+1} in \pref{eq: apply lemma 6}, noticing that $\eta = \frac{1}{20T}$, and taking a union bound over epochs finishes the proof.

\end{proof}

\begin{lemma}
    \label{lem:one-interval-omd}
       With probability at least $1-\delta$, for all epochs $\calI_k=[e_k, t_k]$, choosing $\xx$ ensures
       \begin{align*}
          &\DReg_{k,\xx} \leq \sum_{a\in[A]} q_{e_k}(a) L_k(a) + \frac{5}{2}\frac{A\log(AT/\delta)}{\sum_{a\in[A]} \theta_{e_k}(a) } + \sum_{a\in[A]} q_{e_k}(a)\sum_{t\in \calI_k}\left(\hat{c}_t(a) - c_t(a)\right).  
       \end{align*}
       
\end{lemma}

\begin{proof}
    Below, denote $\tdelta=\frac{\delta}{AT}$. 
    For any $\theta\in\Theta$, 
    \begin{align}
        &\sum_{t\in \calI_k}\sum_{a\in[A]} \theta(a)(c_{t}(a_t) - c_{t}(a)) \notag\\
        &=\sum_{t\in \calI_k}\sum_{a\in[A]} \theta(a)(c_{t}(a_t) - \hat{c}_{t}(a)) + \underbrace{\sum_{t\in \calI_k}\sum_{a\in[A]} \theta(a)(\hat{c}_{t}(a) - c_t(a))}_{:=Z_k(\theta)}  \label{eq: Z deff} \\ 
        &= \sqrt{\frac{A}{T}}\sum_{t\in \calI_k} \inner{\theta, \hat{v}_t}  + Z_k(\theta)  \tag{by the definition of $\vhat_t$} \\
        &\leq \sqrt{\frac{A}{T}}\cdot\frac{D_\psi(\theta, \theta_{e_k})-D_\psi(\theta, \theta_{t_k+1})}{\eta} + \frac{5}{2}A\log(1/\tdelta) + Z_k(\theta).   \tag{\pref{lem: mirror descent for x}}
    \end{align}
    Letting $\theta=\theta_{e_k}$ above and using the fact that $D_\psi (\theta_{e_k}, \theta_{e_k})=0$, we get 
    \begin{align}
\sum_{a\in[A]}\theta_{e_k}(a)\sum_{t\in \calI_k}\left(c_t(a_t) - c_t(a)\right)  \leq \frac{5}{2}A\log(1/\tdelta) + Z_k(\theta_{e_k}). 
    \label{eq: apply theta=ek}
    \end{align}
    Notice that
    \begin{align*}
         &\sum_{a\in[A]}\theta_{e_k}(a)\sum_{t\in \calI_k} \left(c_t(a_t) - \min_{a'\in[A]} c_t(a')\right)  \\
         &= \sum_{a\in[A]}\theta_{e_k}(a)\sum_{t\in \calI_k} \left(c_t(a_t) - c_t(a) + c_t(a) - \min_{a'\in[A]} c_t(a')\right)  \\
         &\leq \frac{5}{2}A\log(1/\tdelta) + Z_k(\theta_{e_k}) + \sum_{a\in[A]}\theta_{e_k}(a) L_k(a). \tag{by \pref{eq: apply theta=ek} and the definition of $L_k(a)$}
    \end{align*}
    Divide both sides by $\sum_{a\in[A]} \theta_{e_k}(a)$:  
    \begin{align*}
         &\sum_{t\in \calI_k} \left(c_t(a_t) - \min_{a'\in[A]} c_t(a')\right) \\
         &\leq \frac{5}{2}\frac{A\log(1/\tdelta)}{\sum_{a\in[A]} \theta_{e_k}(a) } + \frac{Z_k(\theta_{e_k})}{\sum_{a\in[A]} \theta_{e_k}(a) } + \sum_{a\in[A]} q_{e_k}(a) L_k(a)  \\ 
         &= \frac{5}{2}\frac{A\log(1/\tdelta)}{\sum_{a\in[A]} \theta_{e_k}(a) } + \sum_{a\in[A]} q_{e_k}(a)\sum_{t\in \calI_k}\left(\hat{c}_t(a) - c_t(a)\right) + \sum_{a\in[A]} q_{e_k}(a) L_k(a).   \tag{recall the definition of $Z_k(\theta)$ in \pref{eq: Z deff}}
    \end{align*}
\end{proof}

\subsection{In $i_k=\oo$ Epochs (\pref{lem: mirror descent in o})}

\begin{lemma}\label{lem: mirror descent in o}
   With probability at least $1-O(\delta)$, for all epochs $\calI_k = [e_k, t_k]$ with $i_k=\oo$, if holds for any $\theta\in\Theta$, 
   \begin{align*}
       \sum_{t\in\calI_k} \inner{\theta, \vhat_t} 
       &\leq  \frac{D_\psi(\theta, \theta_{e_k}) - D_\psi(\theta, \theta_{t_{k}+1})}{\eta}  \\
       &\qquad \qquad + \frac{1}{up_{k,\oo}\sqrt{S}} \min\left\{8\rho T +\frac{12\sqrt{AT}\hat{s}_k\log(AT/\delta)}{\rho}, \frac{3}{2}T\right\}  + \sum_{t\in \calI_k} \inner{\theta, \tfrac{1}{3}|\vhat_t|}, 
   \end{align*}
   where $|\vhat_t|$ denotes the vector $(|\vhat_t(1)|, \cdots, |\vhat_t(A+1)|)$. 
\end{lemma}

\begin{proof}
   The mirror descent in $i=\oo$ case follows the update in \pref{lem: lem with correction}.  
   First, we verify the conditions for \pref{lem: lem with correction}.  
   Notice that for epochs with $i_k=\oo$, we have $\hat{v}_t(a)=0 $ for all $a\in[A+1]$ except when $t=k\sqrt{T}$. When $t=k\sqrt{T}$, we have for $a\in[A]$, 
   \begin{align}
      \eta |\hat{v}_t(a)| &= \eta  \left| \min\left\{8\rho T + \frac{12\sqrt{AT}\hat{s}_k\log(AT/\delta)}{\rho}, \frac{3}{2}T \right\}- u\sqrt{\frac{ T}{A}}\Lhat_k(a) \right| \notag\\
      &\leq  \frac{3}{2}\eta T + u\eta T \leq \frac{1}{6},  \label{eq: verify 1}  
   \end{align}
   and for $a=A+1$, 
   \begin{align}
       \eta \theta_t(A+1) |\hat{v}_t(A+1)| 
       &= \eta \theta_{e_k}(A+1) \left|\frac{p_{k,\xx}}{p_{k,\oo}} \sqrt{\frac{T}{SA}} \sum_{a'=1}^A q_{e_k}(a')\hat{L}_k(a')\right|\notag \\
       &\leq \eta \theta_{e_k}(A+1) \frac{u\sqrt{S}\sum_{a'=1}^A \theta_{e_k}(a') }{\theta_{e_k}(A+1)} \sqrt{\frac{T}{SA}}  \sqrt{AT} \tag{by \pref{eq: pkopox}} \\
       &\leq u\eta \sum_{a'=1}^A \theta_{e_k}(a') T \leq \frac{1}{6}.  \label{eq: verify 2}
   \end{align}
   Hence, the conditions for \pref{lem: lem with correction} hold. By \pref{lem: lem with correction}, for any $\theta\in\Theta$, it holds that 
   \begin{align}
        \sum_{t\in \calI_k} \inner{\theta, \hat{v}_t}
        &\leq 
         \frac{D_\psi(\theta, \theta_{e_k}) - D_\psi(\theta, \theta_{t_{k}+1})}{\eta}   + \sum_{t\in \calI_k} \inner{\theta_{t}, \hat{v}_t}  + 2\eta \sum_{t\in \calI_k} \inner{\theta, w_t},    \label{eq: oo bound}
    \end{align}
    where 
    \begin{align}
       w_t(a) = \begin{cases}
           \hat{v}_t(a)^2,   &\text{if}\ a\in[A], \\
           \theta_t(a)\hat{v}_t(a)^2,  &\text{if}\ a=A+1.    
       \end{cases}   \label{eq: def of w 100}
    \end{align}
We proceed to bound the right-hand side of \pref{eq: oo bound} below. Denote $\tdelta=\frac{\delta}{AT}$. 
\paragraph{Bounding $\sum_{t\in\calI_k} \inner{\theta_t,\vhat_t}$. }
In an epoch $\calI_k=[e_k, t_k]$ with $i_k=\oo$, $\vhat_t$ is non-zero only when $t=t_k$. This implies $\theta_{e_k}=\theta_{e_k+1}=\cdots =\theta_{t_k}$. Therefore, 
    \begin{align}
       &\sum_{t\in\calI_k}\theta_t^\top \hat{v}_t 
       = \sum_{a=1}^{A+1} \theta_{e_k}(a) \hat{v}_{t_k}(a) \nonumber \\
       &=  \sum_{a=1}^A \theta_{e_k}(a)\left(\min\left\{8\rho T + \frac{12\sqrt{AT}\hat{s}_k\log(1/\tdelta)}{\rho}, \frac{3}{2}T\right\}  - u\sqrt{\frac{T}{A}} \hat{L}_k(a) \right) \notag\\
       &\qquad \qquad + \theta_{e_k}(A+1) \frac{p_{k,\xx}}{p_{k,\oo}}\sqrt{\frac{T}{SA}} \sum_{a=1}^A q_{e_k}(a)\hat{L}_k(a). 
\label{eq: balckwell}
    \end{align}
By the definition of $p_{k,\xx}, p_{k,\oo}$ in \pref{eq: pkopox}, we have 
\begin{align}
    &\theta_{e_k}(A+1) \frac{p_{k,\xx}}{p_{k,\oo}}\sqrt{\frac{T}{SA}} \sum_{a=1}^A q_{e_k}(a)\hat{L}_k(a)    \nonumber \\
    &\leq\theta_{e_k}(A+1) \frac{u\sqrt{S} \sum_{a=1}^A \theta_{e_k}(a)}{\theta_{e_k}(A+1)}\sqrt{\frac{T}{SA}} \sum_{a=1}^A q_{e_k}(a)\hat{L}_k(a)  \nonumber \\
    &= u \sqrt{\frac{T}{A}} \left(\sum_{a=1}^A \theta_{e_k}(a)\right) \left(\sum_{a=1}^A q_{e_k}(a) \hat{L}_k(a)\right)  \nonumber \\
    &= u\sqrt{\frac{T}{A}}  \sum_{a=1}^A  \theta_{e_k}(a) \hat{L}_k(a).   \tag{because $q_{t}(a) = \theta_t(a) / \big(\sum_{a'=1}^A \theta_{t}(a')\big)$}   \\
    & \label{eq: cancelation}
\end{align}
Using \pref{eq: cancelation} in \pref{eq: balckwell}, we get 
\begin{align}
    \sum_{t\in\calI_k} \inner{\theta_t,\vhat_t} \leq \left(\sum_{a=1}^A \theta_{e_k}(a)\right)  \times \min\left\{8\rho T + \frac{12\sqrt{AT}\hat{s}_k\log(1/\tdelta)}{\rho}, \frac{3}{2}T\right\}.\label{eq: tmp bound oo} 
\end{align}
By the definition of $p_{k,\xx}, p_{k,\oo}$ in \pref{eq: pkopox}, we have either 
\begin{align*}
    \sum_{a=1}^A \theta_{e_k}(a) = \frac{1}{u}\frac{p_{k,\xx}}{p_{k,\oo}} \frac{\theta_{e_k}(A+1)}{\sqrt{S}}  \qquad \text{or} \qquad p_{k,\oo} = \frac{1}{\sqrt{S}}.   
\end{align*}
In the former case, the right-hand side of \pref{eq: tmp bound oo}  is equal to 
    \begin{align*}
        &\frac{1}{u}\frac{p_{k,\xx}}{p_{k,\oo}} \frac{\theta_{e_k}(A+1)}{\sqrt{S}}\times \min\left\{8\rho T + \frac{12\sqrt{AT}\hat{s}_k\log(1/\tdelta)}{\rho}, \frac{3}{2}T\right\}  \\
        &\qquad \qquad \leq \frac{1}{up_{k,\oo}\sqrt{S}} \times \min\left\{8\rho T + \frac{12\sqrt{AT}\hat{s}_k\log(1/\tdelta)}{\rho}, \frac{3}{2}T \right\}. 
    \end{align*}
    In the latter case, the right-hand side of \pref{eq: tmp bound oo}  is upper bounded by 
    \begin{align*}
       \min\left\{8\rho T + \frac{12\sqrt{AT}\hat{s}_k\log(1/\tdelta)}{\rho}, \frac{3}{2}T \right\} \leq \frac{1}{up_{k,\oo}\sqrt{S}}\times \min\left\{8\rho T + \frac{12\sqrt{AT}\hat{s}_k\log(1/\tdelta)}{\rho}, \frac{3}{2}T \right\}.
    \end{align*}
\paragraph{Bounding $2\eta\sum_{t\in\calI_k} \inner{\theta, w_t}$. }
By the definition of $w_t$, we have 
    \begin{align*}
         2\eta \sum_{t\in \calI_k} \inner{\theta, w_t} 
         &= 2\eta \sum_{t\in\calI_k} \left(\sum_{a\in[A]}  \theta(a) \vhat_t(a)^2 + \theta(A+1)\theta_t(A+1)\vhat_t(A+1)^2\right)\\
         &\leq \frac{1}{3} \sum_{t\in\calI_k} \left(\sum_{a\in[A]}\theta(a)|\vhat_t(a)| + \theta(A+1)|\vhat_t(A+1)|\right) \tag{by \pref{eq: verify 1} and \pref{eq: verify 2}} \\
         &= \frac{1}{3} \sum_{t\in \calI_k}  \inner{\theta, |\hat{v}_t|}.     
    \end{align*}

\end{proof}

\section{Relating Regrets to $\vhat_t$ (\pref{lemma: connect sreg to v}, \pref{lem: connect dreg to v})}

\begin{lemma}\label{lemma: connect sreg to v} With probability at least $1-O(\delta)$, for all $a\in[A]$, 
\begin{align}
   \SReg(a) &\leq  \sqrt{AT} + \sqrt{\frac{A}{T}}\cdot \sum_{k=1}^K 
\sum_{t\in \calI_k} \left(\hat{v}_t(a) \ind\{i_k=\xx\}  +  \left(\hat{v}_t(a) - \tfrac{1}{3}|\hat{v}_t(a)|\right) \ind\{i_k=\oo\}\right). \label{eq: }
\end{align}
\end{lemma}

\begin{proof}
By the definition of $\SReg_{k,\xx}(a)$,
\begin{align}
    \SReg_{k,\xx}(a) \ind\{i_k=\xx\}  
    &= \sum_{t\in \calI_k} \left(c_t(a_t) - c_t(a)\right) \ind\{i_k=\xx\}.   \label{eq: SRegx 1}
\end{align}
Below, denote $\io=\log(AT/\delta)$. By \pref{lemma: sreg under ooo}, 
\begin{align}
    \SReg_{k,\oo}(a)\ind\{i_k=\oo\} 
    &\leq \left(  \min\left\{5\rho\sqrt{AT} + \frac{7A\hat{s}_k\io}{\rho}, \sqrt{AT}\right\} - \frac{3}{4}\Lhat_k(a) \right) \ind\{i_k=\oo\}. \label{eq: SRego 1} 
\end{align}
Thus, 
\begin{align*}
    &\SReg(a) \\
    &= \sum_{k=1}^{K} \SReg_{k,\xx}(a) \ind\{i_k=\xx\} + \sum_{k=1}^K \SReg_{k,\oo}(a)\ind\{i_k=\oo\} \\
    &= \sum_{k=1}^{K} \sum_{t\in \calI_k}\left(c_t(a_t) - c_t(a)\right) \ind\{i_k=\xx\} + \sum_{k=1}^K \SReg_{k,\oo} \ind\{i_k=\oo\}  \tag{by \pref{eq: SRegx 1}} \\
    &\leq \sum_{k=1}^{K} \sum_{t\in \calI_k}\left(c_t(a_t) - \hat{c}_t(a)\right) \ind\{i_k=\xx\} + \sum_{k=1}^{K} \sum_{t\in \calI_k}\left(\hat{c}_t(a) - c_t(a)\right) \ind\{i_k=\xx\} \\
    &\qquad \quad + \sum_{k=1}^K \left(  \min\left\{ 5\rho\sqrt{AT} +\frac{7A\hat{s}_k\io}{\rho}, \sqrt{AT}\right\} - \frac{3}{4}\Lhat_k(a)\right) \ind\{i_k=\oo\} \tag{by \pref{eq: SRego 1}} \\
    &=  \sum_{k=1}^{K} \sum_{t\in \calI_k}\left(c_t(a_t) - \hat{c}_t(a)\right) \ind\{i_k=\xx\} + \frac{\io}{2\gamma}    \tag{\pref{lem: IX lemma}}\\
    &\quad \qquad + \sum_{k=1}^K \left(\min\left\{5\rho \sqrt{AT} + \frac{7\hat{s}_k A \io}{\rho}, \sqrt{AT}\right\}  -\frac{3}{4}\hat{L}_k(a)\right) \ind\{i_k=\oo\} \\
    &\leq \sqrt{\frac{A}{T}}\cdot \sum_{k=1}^K 
\sum_{t\in \calI_k} \hat{v}_t(a) \ind\{i_k=\xx\} +  \sqrt{\frac{A}{T}}\cdot \sum_{k=1}^K 
\sum_{t\in \calI_k} \left(\hat{v}_t(a) - \frac{1}{3}|\hat{v}_t(a)|\right) \ind\{i_k=\oo\} + \sqrt{AT}. \tag{by the definition of $\hat{v}_t$}   
\end{align*}
\end{proof}

\begin{lemma}\label{lem: connect dreg to v}
    With probability at least $1-O(\delta)$, 
    \begin{align*}
        \DReg\leq  6\sqrt{\frac{S A}{T}}\sum_{k=1}^K \sum_{t\in\calI_k}\big(
 \hat{v}_t(A+1)\ind\{i_k=\xx\} + (\hat{v}_t(A+1) -\tfrac{1}{3} |\hat{v}_t(A+1)|)\ind\{i_k=\oo\}\big) + 80\sqrt{SAT}\io,  
    \end{align*}
    where $\io=\log(AT/\delta)$. 
\end{lemma}

\begin{proof}
    \begin{align}
        \DReg
        &= \sum_{k=1}^K \DReg_{k,\xx} \ind\{i_k=\xx\} + \sum_{k=1}^K \DReg_{k,\oo}\ind\{i_k=\oo\} \notag \\
        &\leq \underbrace{\frac{5}{2}\sum_{k=1}^K \ind\{i_k=\xx\} \frac{A\io}{\sum_{a\in[A]} \theta_{e_k}(a) }}_{\term_1} \notag\\
        &\qquad + \underbrace{\sum_{k=1}^K \ind\{i_k=\xx\}\sum_{a\in[A]} q_{e_k}(a) L_k(a) }_{\term_2} \notag\\
        &\qquad + \underbrace{\sum_{k=1}^K \ind\{i_k=\xx\} \sum_{a\in[A]} q_{e_k}(a)\sum_{t\in \calI_k}\left(\hat{c}_t(a) - c_t(a)\right)}_{\term_3}  \tag{\pref{lem:one-interval-omd}}\\
        &\qquad + \underbrace{\sum_{k=1}^K \ind\{i_k=\oo\}\left(5\rho\sqrt{AT} + \frac{10A(\hat{s}_k + \ind\{s_k > 0\})\io}{\rho}\right)}_{\term_4}. \tag{\pref{lem: dreg in o}} \\
        \label{eq: dyna regret decomp} 
\end{align}
We bound the four parts in \pref{eq: dyna regret decomp} individually below.  

\paragraph{Bounding $\term_1$. }
\begin{align*}
    \frac{5}{2}\sum_{k=1}^K \ind\{i_k=\xx\} \frac{A\io}{\sum_{a\in[A]} \theta_{e_k}(a) } &\leq \frac{15}{4}\sum_{k=1}^K p_{k,\xx} \frac{A\io}{\sum_{a\in[A]} \theta_{e_k}(a) } + \frac{10A\io^2}{A\gamma}   \tag{\pref{lem: freedman+} and that $\sum_{a\in[A]} \theta_{e_k}(a)\geq A\gamma$ by \pref{line: clipping definition feas}}\\
    &\leq \frac{15}{4}\sum_{k=1}^K  \frac{A\io}{\sum_{a\in[A]} \theta_{e_k}(a) + 
 \frac{1}{u\sqrt{S}}\theta_{e_k}(A+1) } + \frac{5}{2}\sqrt{AT}\io \tag{by the definition of $p_{k,\xx}$} \\
 &\leq \frac{15}{4}\sqrt{\frac{T}{A}}\cdot A\sqrt{S}\io + \frac{5}{2}\sqrt{AT} \io\tag{using $\theta_{e_k}\in\Delta_{A+1}$ and $u\leq 1$}\\
 &\leq 7 \sqrt{SAT}\io. 
\end{align*}
\paragraph{Bounding $\term_2$. }
\begin{align*}
    &\sum_{k=1}^K \ind\{i_k=\xx\}\sum_{a\in[A]} q_{e_k}(a) L_k(a) \\
    &\leq 2 \sum_{k=1}^K p_{k,\xx}\sum_{a\in[A]} q_{e_k}(a) L_k(a) + 4\sqrt{AT}\io  \tag{\pref{lem: freedman+}}  \\
    &\leq 4 \sum_{k=1}^K 
\frac{p_{k,\xx} \ind\{i_k=\oo\}}{p_{k,\oo}} \sum_{a\in[A]} q_{e_k}(a)L_k(a) + 4\sqrt{AT}\io +  8\times 2\sqrt{S}\times \sqrt{AT}\io\tag{by \pref{lem: freedman+} and that $p_{k,\oo}\geq \frac{1}{\sqrt{S}}$} \\
&\leq 4 \sum_{k=1}^K 
\frac{p_{k,\xx} \ind\{i_k=\oo\}}{p_{k,\oo}} \sum_{a\in[A]} q_{e_k}(a)\Lhat_k(a)  \\
&\qquad + 4 \sum_{k=1}^K 
\frac{p_{k,\xx} \ind\{i_k=\oo\}}{p_{k,\oo}} \sum_{a\in[A]} q_{e_k}(a)(L_k(a)-\Lhat_k(a)) + 20\sqrt{SAT}\io \\
&\leq 4\sum_{t=1}^T \sqrt{\frac{SA}{T}}\vhat_t(A+1)     + \sum_{k=1}^K \frac{p_{k,\xx} \ind\{i_k=\oo\}}{
p_{k,\oo}} \min\left\{A + \frac{6As_k\io}{\rho}, \sqrt{AT}\right\} + 20\sqrt{SAT}\io \tag{by the definition of $\vhat_t(A+1)$ and \pref{lem: L error} and $\hat{s}_k\leq s_k$}\\
&\leq 4 \sum_{t=1}^T \sqrt{\frac{SA}{T}}\vhat_t(A+1)  + 2 \sum_{k=1}^K p_{k,\xx} \left(A + \frac{6As_k\io}{\rho}\right) + 4\sqrt{S}\sqrt{AT}\io + 20\sqrt{SAT}\io \tag{by \pref{lem: freedman+} and that $p_{k,\oo}\geq \frac{1}{\sqrt{S}}$}\\
&\leq 4\sum_{t=1}^T \sqrt{\frac{SA}{T}}\vhat_t(A+1) + 2\sqrt{AT} + \frac{12SA\io}{\rho} + 20\sqrt{SAT}\io \\
&\leq 4\sum_{t=1}^T \sqrt{\frac{SA}{T}}\vhat_t(A+1) + 25\sqrt{SAT}\io. 
\end{align*}
\paragraph{Bounding $\term_3$. } By \pref{lem: IX lemma}, 
\begin{align*}
    \sum_{k=1}^K \ind\{i_k=\xx\} \sum_{a\in[A]} q_{e_k}(a)\sum_{t\in \calI_k}\left(\hat{c}_t(a) - c_t(a)\right) \leq \frac{\io}{2\gamma} = \frac{1}{2}\sqrt{AT}.  
\end{align*}
\paragraph{Bounding $\term_4$. } As $\hat{s}_k\leq s_k$, we have 
\begin{align*}
    \term_4\leq \sum_{k=1}^K \ind\{i_k=\oo\}\left(5\rho\sqrt{AT} + \frac{20As_k\io}{\rho}\right)\leq 5\rho T + \frac{20SA\io}{\rho}.  
\end{align*}
Plugging the upper bounds for individual terms in \pref{eq: dyna regret decomp} and using $\rho=4\io\sqrt{\frac{SA}{T}}$ we have 
   \begin{align*}
        \DReg\leq 4\sqrt{\frac{S A}{T}}\sum_{t=1}^T 
\hat{v}_t(A+1)  + 80\sqrt{SAT}\io. 
    \end{align*}
    Then using the fact that $\hat{v}_t(A+1)\geq 0$ in $i_k=\oo$ epochs, and $\hat{v}_t(A+1)=0$ in $i_k=\xx$ epochs completes the proof. 
\end{proof}


\section{Upper Bounding Sum of $\vhat_t$ (\pref{lem: final regret bound})}

\begin{lemma}\label{lem: final regret bound}
   With probability at least $1-O(\delta)$, for any $a\in[A+1]$, 
   \begin{align*}
        \sum_{t=1}^T \hat{v}_t(a) &\leq \frac{1}{3} \sum_{k=1}^K \sum_{t\in \calI_k}\ind\{ i_k=\oo \} |\vhat_t(a)|  + 400 T\io,  
   \end{align*}
   where $\io=\log(AT/\delta)$. 
\end{lemma}
\begin{proof}
    Combining \pref{lem: mirror descent for x} and \pref{lem: mirror descent in o}, we have for any $\theta\in\Theta$, 
    \begin{align}
        &\sum_{t=1}^T  \inner{\theta, \vhat_t} 
        = \sum_{k=1}^K \sum_{t\in\calI_k}  \inner{\theta, \vhat_t} \notag \\
        &\leq \frac{D_\psi(\theta, \theta_1)}{\eta} + \frac{5}{2}\sum_{k=1}^K \ind\{i_k=\xx\}\sqrt{AT}\io \notag\\
        &\qquad + \underbrace{\sum_{k=1}^K  \frac{\ind\{i_k=\oo\}}{up_{k,\oo}\sqrt{S}}\left(\min\left\{8\rho T + \frac{12\sqrt{AT}\hat{s}_k \io}{\rho}, \frac{3}{2}T\right\} \right)}_{(\star)} + \frac{1}{3}\sum_{k=1}^K \sum_{t\in\calI_k} \inner{\theta, |\vhat_t|}\ind\{i_k=\oo\} \label{eq: tmp bound for main}
    \end{align}
    The $(\star)$ term can be bounded as 
\begin{align*}
   (\star)&\leq  \frac{16}{9}\sum_{k=1}^K  \frac{\ind\{i_k=\oo\}}{p_{k,\oo}\sqrt{S}}\min\left\{8\rho T + \frac{12\sqrt{AT}s_k \io}{\rho}, \frac{3}{2}T\right\}   \tag{$\hat{s}_k\leq s_k$}\\
   &\leq 3\sum_{k=1}^K  \frac{1}{\sqrt{S}}\min\left\{8\rho T + \frac{12\sqrt{AT}s_k \io}{\rho}, \frac{3}{2}T\right\} + 8\times \frac{3}{2}T \io   \tag{using \pref{lem: freedman+} and that $p_{k,\oo}\sqrt{S}\geq 1$}\\
        &\leq 3\frac{1}{\sqrt{S}}\left(32K\sqrt{SAT}\io + \frac{12\sqrt{AT}S\io}{4\io\sqrt{\frac{SA}{T}}} \right) + 12T\io \tag{$\exprate = 4\io\sqrt{SA/T}$}\\
        &\leq 120T\io.   
\end{align*}
    For a fixed $a\in[A+1]$, consider 
    \begin{align*}
        \theta = (1-\alpha) \e_a + \alpha \theta_1. 
    \end{align*}
    Then \pref{eq: tmp bound for main} gives 
    \begin{align}
         \sum_{t=1}^T \hat{v}_t(a) &\leq \frac{D_\psi(\theta, \theta_1)}{\eta} + \frac{1}{3} \sum_{k=1}^K \sum_{t\in \calI_k}\ind\{ i_k=\oo \} |\vhat_t(a)|  +  122.5T\io  \nonumber \\
         &\qquad \qquad + \alpha  \sum_{t=1}^T  \left\langle  
\e_a - \theta_1, \hat{v}_t \right\rangle + \frac{\alpha}{3} \sum_{k=1}^K \ind\{i_k=\oo\} \sum_{t\in \calI_k} \left\langle  \theta_1 - \e_a, |\vhat_t| \right\rangle.   \label{eq: gives} 
    \end{align}
    Below, we bound the remaining terms above. 
    \begin{align*}
        D_\psi(\theta, \theta_1) 
        &= \psi(\theta) - \psi(\theta_1) - \inner{\nabla\psi(\theta_1), \theta-\theta_1}  \\
        &\leq \psi(\theta) - \psi(\theta_1)   \tag{$\theta_1$ is the minimizer of $\psi$} \\
        &= \left(  \ln \frac{1}{\theta(A+1)} + \ln \frac{1}{\sum_{a'=1}^A \theta(a') } + \sum_{a=1}^{A+1}  \theta(a) \ln \theta(a) \right)  \\
        &\qquad \qquad - \left(  \ln \frac{1}{\theta_1(A+1)} + \ln \frac{1}{\sum_{a'=1}^A \theta_1(a') } + \sum_{a=1}^{A+1}  \theta_1(a) \ln \theta_1(a) \right) \\
        &\leq  \ln \frac{\theta_1(A+1)}{\theta(A+1)} + \ln \frac{\sum_{a'=1}^A \theta_1(a') }{\sum_{a'=1}^A \theta(a')} - \sum_{a=1}^{A+1}  \theta_1(a)\ln\theta_1(a)   \tag{$\sum_a \theta(a)\ln\theta(a)\leq 0$} \\
        &\leq 2\ln(1/\alpha) + \ln(A+1).   \tag{$\theta(a)\geq \alpha\theta_1(a)$, and Shannon entropy $\leq \ln(A+1)$}
    \end{align*}
\paragraph{$a\in[A]$ case. }
For $a\in[A]$, let $\alpha=\frac{1}{T}$, which makes $\theta\in\Theta$ (defined in \pref{line: clipping definition feas}). We bound 
    \begin{align*}
        &\alpha  \sum_{t=1}^T  \left\langle  
\e_a - \theta_1, \hat{v}_t \right\rangle + \frac{\alpha}{3} \sum_{k=1}^K \ind\{i_k=\oo\} \sum_{t\in \calI_k} \left\langle  \theta_1 - \e_a, |\vhat_t| \right\rangle  \\
        &\leq \frac{4\alpha}{3} \sum_{t=1}^T  20T  \leq 27T.  \tag{$|\hat{v}_t(a)|\leq 10T$ for any $t$}
    \end{align*}
    \paragraph{$a=A+1$ case. } For $a=A+1$, let $\alpha=2A\gamma = 8\sqrt{\frac{A}{T}}\io$, which also makes $\theta\in\Theta$ because $\sum_{a\in[A]}\theta(a)\geq \alpha \sum_{a\in[A]}\theta_1(a)\geq \frac{1}{2}\alpha=A\gamma$ by \pref{lem: at least 1/2}. We bound 
    \begin{align*}
        &\alpha  \sum_{t=1}^T  \left\langle  
\e_a - \theta_1, \hat{v}_t \right\rangle + \frac{\alpha}{3} \sum_{k=1}^K \ind\{i_k=\oo\} \sum_{t\in \calI_k} \left\langle  \theta_1 - \e_a, |\vhat_t| \right\rangle  \\
        &=\alpha \sum_{k=1}^K \sum_{t\in\calI_k} \left\langle  
\e_a - \theta_1, \hat{v}_t \right\rangle \ind\{i_k=\xx\} + \alpha \sum_{k=1}^K \sum_{t\in\calI_k} \left\langle  
\e_a - \theta_1, \vhat_t - \tfrac{1}{3}|\vhat_t|\right\rangle \ind\{i_k=\oo\} \\
&\leq - \alpha \sum_{k=1}^K \sum_{t\in\calI_k} \left\langle  
 \theta_1, \hat{v}_t \right\rangle \ind\{i_k=\xx\} + \frac{4\alpha}{3} \sum_{k=1}^K 20T \\
 &\leq -\alpha \sum_{k=1}^K \sum_{t\in\calI_k} \left\langle  
 \theta_1, \hat{v}_t \right\rangle \ind\{i_k=\xx\} + 240T\io, 
    \end{align*} 
    where in the second-to-last inequality we use that $\vhat_t(A+1)=0$ when $i_k=\xx$, and that $\|\vhat_t\|_\infty\leq 10T$ and $\vhat_t\neq 0$ in only one round in an epoch with $i_k=\oo$.  Finally, we bound the first term in the last expression: 
    \begin{align*}
        &-\alpha \sum_{k=1}^K  \ind\{i_k=\xx\} \sum_{t\in\calI_k} \inner{\theta_1, \vhat_t} \\
        &= \alpha \sqrt{\frac{T}{A}}\sum_{a=1}^A \theta_1(a) \sum_{k=1}^K \ind\{i_k=\xx\} \sum_{t\in\calI_k} \left(\hat{c}_t(a) - c_t(a_t)\right) \\
        &=  8\io\sum_{a=1}^A \theta_1(a) \sum_{k=1}^K \ind\{i_k=\xx\} \sum_{t\in\calI_k} \left(\hat{c}_t(a) - c_t(a)\right) \\
        &\qquad +  8\io\sum_{a=1}^A \theta_1(a) \sum_{k=1}^K \ind\{i_k=\xx\} \sum_{t\in\calI_k} \left(c_t(a) - c_t(a_t)\right)\tag{$\alpha=8\sqrt{\frac{A}{T}}\io$} \\
        &\leq  \frac{4\io^2}{\gamma} + 8T\io \tag{\pref{lem: IX lemma}} \\
        &\leq 12T\io.  
    \end{align*}
    Collecting terms gives the desired bound. 
\end{proof}

\begin{lemma}\label{lem: at least 1/2}
    \begin{align*}
        \sum_{a=1}^A \theta_1(a)\geq \frac{1}{2}.  
    \end{align*}
\end{lemma}
\begin{proof}
     Let $\tilde{\theta}$ be the minimizer of $\psi(\theta) = \sum_{a\in[A+1]}\theta(a)\ln \theta(a) + \log\frac{1}{\sum_{a\in[A]}\theta(a)} + \log\frac{1}{\theta(A+1)}$ over $\Delta_{A+1}$. Then the KKT condition requires 
     \begin{align*}
         &1 + \log\tilde{\theta}(a) - \frac{1}{\sum_{a\in[A]}\tilde{\theta}(a)} + \lambda =0, \qquad \text{for}\ a\in[A], \\
         &1 + \log\tilde{\theta}(A+1) - \frac{1}{\tilde{\theta}(A+1)} + \lambda =0
     \end{align*}
     for some $\lambda\in\mathbb{R}$. 
     Suppose that $\sum_{a=1}^A \tilde{\theta}(a)<\frac{1}{2}$. Then we have $\tilde{\theta}(A+1) > \frac{1}{2}$, and thus for any $a\in[A]$
     \begin{align*}
         1 + \log\tilde{\theta}(a) - \frac{1}{\sum_{a\in[A]}\tilde{\theta}(a)} + \lambda 
         &< 1 +\log\frac{1}{2} - \frac{1}{1/2} + \lambda \\
         &< 1+ \ln \tilde{\theta}(A+1) - \frac{1}{\tilde{\theta}(A+1)} + \lambda, 
     \end{align*}
     contradicting with the KKT condition. Therefore, $\sum_{a=1}^A \tilde{\theta}(a)\geq \frac{1}{2}$. As $\tilde{\theta}\in \Theta\subset \Delta_{A+1}$, $\tilde{\theta}$ is also the minimizer of $\psi$ over $\Theta$, which is $\theta_1$. Hence, we have $\sum_{a=1}^A \theta_1(a)\geq \frac{1}{2}$.


\end{proof}

\end{document}